\documentclass{article}
\usepackage[nonatbib,preprint]{neurips_2020}
\usepackage[utf8]{inputenc} % allow utf-8 input
\usepackage[T1]{fontenc}    % use 8-bit T1 fonts
\usepackage{hyperref}       % hyperlinks
\usepackage{url}            % simple URL typesetting
\usepackage{booktabs}       % professional-quality tables
\usepackage{amsfonts}       % blackboard math symbols
\usepackage{nicefrac}       % compact symbols for 1/2, etc.
\usepackage{microtype}      % microtypography

\usepackage{scalerel}
\usepackage{amsthm}
\usepackage{amssymb}
\usepackage{mathtools}
\usepackage{amsmath}
\usepackage{txfonts}
\usepackage{fontawesome}
\usepackage{wrapfig}

\usepackage{todonotes}
\usepackage{xcolor}
\usepackage{setspace}
\usepackage{multirow}
\usepackage{bm,bbm}
\usepackage{paralist}
\usepackage{enumitem}

\newtheorem{corollary}{Corollary}
\newtheorem{problem*}{Problem}
\newtheorem{property}{Property}
\newtheorem{theorem}{Theorem}
\newtheorem{lemma}{Lemma}

\newtheorem{definition}{Definition}

\newtheorem{example}{Example}
\newtheorem*{example*}{Example}

\usepackage[ruled,linesnumbered,noresetcount,vlined]{algorithm2e}
\makeatletter

\newcommand{\nosemic}{\renewcommand{\@endalgocfline}{\relax}}% Drop semi-colon ;
\newcommand{\dosemic}{\renewcommand{\@endalgocfline}{\algocf@endline}}% Reinstate semi-colon ;
\let\oldnl\nl% Store \nl in \oldnl
\newcommand{\nonl}{\renewcommand{\nl}{\let\nl\oldnl}}% Remove line number for one line

\SetKwFor{Repeat}{repeat}{:}{endw}
\makeatother

\newcommand{\Lap}{\textrm{Lap}}

\DeclareMathOperator*{\argmin}{argmin}

\newcommand*{\prob}{\mathsf{P}}
\newcommand\independent{\protect\mathpalette{\protect\independenT}{\perp}}
\def\independenT#1#2{\mathrel{\rlap{$#1#2$}\mkern2mu{#1#2}}}

\newcommand{\nando}[1]{\todo[inline,caption={},color=cyan!20!]{
\begin{spacing}{0.6}
\vspace{-2pt}
{\footnotesize FF: #1}
\vspace{-8pt}
\end{spacing}}}

\newcommand{\rev}[1]{{\color{purple}{#1}}}

\newcommand*{\defeq}{\stackrel{\text{def}}{=}}

\newcommand{\cM}{\mathcal{M}} 
 
\newcommand{\cR}{\mathcal{R}}

 \newcommand{\cX}{\mathcal{X}}
\newcommand{\bx}{\bm{x}}

\newcommand{\EE}{\mathbb{E}} \newcommand{\RR}{\mathbb{R}}
\newcommand{\NN}{\mathbb{N}}

\newcommand{\var}{\mathrm{Var}}

\usepackage{esvect}
\DeclareMathOperator{\Tr}{Tr}

\DeclarePairedDelimiter\floor{\lfloor}{\rfloor}

\title{Decision Making with Differential Privacy under a Fairness Lens}

\author{%
  Ferdinando Fioretto \\
  {\small Syracuse University} \\
  \texttt{ffiorett@syr.edu}\\
  \And
  Cuong Tran\\
  {\small Syracuse University}\\
  \texttt{cutran@syr.edu}\\
  \And
  Pascal Van Hentenryck\\
  {\small Georgia Institute of Technology}\\
  \texttt{pvh@isye.gatech.edu}
}

\begin{document}
\maketitle\sloppy\allowdisplaybreaks

\begin{abstract}
Agencies, such as the U.S. Census Bureau, release data sets and
statistics about groups of individuals that are used as input to a
number of critical decision processes. To conform with privacy and
confidentiality requirements, these agencies are often required to
release privacy-preserving versions of the data. This paper studies
the release of differentially private data sets and analyzes
their impact on some critical resource allocation tasks under a
fairness perspective.  {The paper shows that, when the decisions take
as input differentially private data}, the noise added to achieve
privacy disproportionately impacts some groups over others. The paper
analyzes the reasons for these disproportionate impacts and
proposes guidelines to mitigate these effects. The proposed approaches 
are evaluated on critical decision problems that use differentially 
private census data.
\end{abstract}

%%%%%%%%%%%%%%%%%%%%%%%%%%%%%%%%%%%%%%%%%%%%%%%%%%%%%%%%%%%%%%%%%%%%%
\section{Introduction}
\label{sec:Introduction}
%%%%%%%%%%%%%%%%%%%%%%%%%%%%%%%%%%%%%%%%%%%%%%%%%%%%%%%%%%%%%%%%%%%%%
Many agencies or companies release statistics about groups of
individuals that are often used as inputs to critical decision
processes.  The U.S.~Census Bureau, for example, releases data that is then used to allocate funds and distribute critical resources to
states and jurisdictions. These decision processes may determine whether a jurisdiction must provide
language assistance during elections, establish distribution plans of
COVID-19 vaccines for states and jurisdictions \cite{covid}, and
allocate funds to school districts \cite{pujol:20,Fioretto:AIJ21}. The resulting 
decisions may have significant societal, economic, and  medical impacts for participating individuals. 

In many cases, the released data contain sensitive information whose
privacy is strictly regulated.  For example, in the U.S., the census
data is regulated under Title 13 \cite{title13}, which requires that
no individual be identified from any data released by the Census
Bureau. In Europe, data release is regulated according to
the \emph{General Data Protection Regulation}
\cite{GDPR}, which addresses the control and transfer of personal data.
As a result, such data releases must necessarily rely on
privacy-preserving technologies. Differential Privacy
(DP) \cite{dwork:06} has become the paradigm of choice for protecting
data privacy, and its deployments have been growing rapidly in the
last decade. These include several data products related to the 2020 release of the
US.~Census Bureau \cite{abowd2018us}, Apple \cite{apple},
Google \cite{erlingsson2014rappor}, and Uber \cite{uber}, and LinkedIn \cite{rogers2020linkedin}.

Although DP provides strong privacy guarantees on the released
data, it has become apparent recently that {\em differential privacy 
may induce biases and fairness issues in downstream decision
processes}, as shown empirically by Pujol et al.~\cite{pujol:20}. Since at least \$675 billion 
are being allocated based on U.S.~census data \cite{pujol:20}, the use of 
differential privacy without a proper understanding of these biases and 
fairness issues may adversely affect the health, well-being, and sense 
of belonging of many individuals. 
Indeed, the allotment of federal funds, apportionment of congressional 
seats, and distribution of vaccines and therapeutics should ideally 
be fair and unbiased. 
Similar issues arise in several other areas including, for instance, 
election, energy, and food policies. The problem is further exacerbated by 
the recent recognition that {\em commonly adopted differential privacy 
mechanisms for data release tasks may in fact introduce unexpected biases 
on their own, independently of a downstream decision process \cite{Zhu:AAAI21}.}

% Statistical agencies thus release \emph{privacy-preserving} data and
% statistics that conform to privacy and confidentiality requirements.
% In the U.S., a small number of decisions, such as congressional
% apportionment, are taken using unprotected true values, but the vast
% majority of decisions rely on privacy-preserving data. Of particular
% interest are resource allocation decisions relying on the U.S.~Census
% Bureau data, since the bureau will release several privacy-preserving 
% data products using the framework of \emph{Differential Privacy} \cite{abowd2018us}
% for their 2020 release.  Recently, \cite{pujol:20} empirically showed
% that differential privacy may have a disparate impact on several
% resource allocation problems. The noise introduced by the privacy
% mechanism may result in decisions that impact various groups
% differently.

This paper builds on these empirical observations and provides a 
step towards a deeper understanding of the fairness issues arising when
differentially private data is used as input to several resource
allocation problems.  {\em One of its main results is to prove that
  several allotment problems and decision rules with significant
  societal impact (e.g., the allocation of educational funds, the
  decision to provide minority language assistance on election
  ballots, or the distribution of COVID-19 vaccines) exhibit inherent
  unfairness when applied to a differentially private release of the census
  data.} To counteract this negative results, the paper examines the
conditions under which decision making is fair when using differential
privacy, and techniques to bound unfairness. The paper also provides a
number of mitigation approaches to alleviate biases introduced by
differential privacy on such decision making problems. More
specifically, the paper makes the following contributions:
\begin{enumerate}[leftmargin=*,labelsep=2pt,itemsep=0pt,parsep=2pt,topsep=2pt]

\item It formally defines notions of fairness and bounded fairness for decision making
  subject to privacy requirements. 

\item It characterizes decision making problems that are fair or
  admits bounded fairness. In addition, it investigates the
  composition of decision rules and how they impact bounded fairness.

\item It proves that several decision problems with high societal impact
  induce inherent biases when using a differentially private input. 

\item It examines the roots of the induced unfairness by analyzing the structure
  of the decision making problems. 

\item It proposes several guidelines to mitigate the negative
  fairness effects of the decision problems studied.  
\end{enumerate}

To the best of the authors' knowledge, this is the first study that
attempt at characterizing the relation between differential privacy
and fairness in decision problems. 

%%%%%%%%%%%%%%%%%%%%%%%%%%%%%%%%%%%%%%
\iffalse
\rev{
Each year, the US federal government allocates at least \$675 billion 
dollars to states and local entities.}

This is, to the best of our knoweledge, the first study of the
fairness impact caused by differentially private decision making. 

In particular, the paper studies the negative impact of differential
privacy to the fairness of decision making under three angles: the
choice of mapping function, the choice of mechanism, and the fairness
evaluation metrics.
\begin{enumerate}[leftmargin=*,labelsep=2pt,itemsep=0pt,parsep=2pt,topsep=2pt]
\item First and foremost, we show that even when the inputs are
treated equally, i.e they were perturbed by the same level of noise,
under certain mapping functions, their outcomes can be unfair. We
study the effect of two set of mapping functions, (1) differentiable
mapping function for continuous input and output, and (2) Boolean
function of binary inputs. 
    
    \item We argue that the choice of private mechanism can be another
source of unfairness. As examples, we will highlight the difference
between input against output perturbation. We then show that applying
an output perturbation mechanism where we add noise directly to the
outcome can ensure fairness, but the accuracy can be suffered.  
\end{enumerate}
\fi
%%%%%%%%%%%%%%%%%%%%%%%%%%%%%%%%%%%%%%

%%%%%%%%%%%%%%%%%%%%%%%%%%%%%%%%%%%%%%%%%%%%%%%%%%%%%%%%%%%%%%%%%%%%%
% \section{Related Work}
%%%%%%%%%%%%%%%%%%%%%%%%%%%%%%%%%%%%%%%%%%%%%%%%%%%%%%%%%%%%%%%%%%%%%

%%%%%%%%%%%%%%%%%%%%%%%%%%%%%%%%%%%%%%%%%%%%%%%%%%%%%%%%%%%%%%%%%%%%%
\section{Preliminaries: Differential Privacy}
%%%%%%%%%%%%%%%%%%%%%%%%%%%%%%%%%%%%%%%%%%%%%%%%%%%%%%%%%%%%%%%%%%%%%

\emph{Differential Privacy} \cite{dwork:06} (DP) is a rigorous privacy notion that characterizes the amount of information of an individual's data being disclosed in a computation.
% Informally, it states that the probability of any differentially private output does not change much when a single individual data is added or removed to the data set, limiting the amount of information that the output reveals about any individual.

\begin{definition}%[Differential Privacy \cite{Dwork:06}]
  A randomized algorithm $\cM:\cX \to \cR$ with domain $\cX$ and range $\cR$ satisfies $\epsilon$-\emph{differential privacy} if
  for any output $O \subseteq \cR$ and data sets $\bm{x}, \bm{x}' \in \cX$ differing by at most one entry (written $\bm{x} \sim \bm{x}'$) 
  \begin{equation}
  \label{eq:dp}
    \Pr[\cM(\bm{x}) \in O] \leq \exp(\epsilon) \Pr[\cM(\bm{x}') \in O]. 
  \end{equation}
\end{definition}

\noindent 
Parameter $\epsilon \!>\! 0$ is the \emph{privacy loss}, with values close 
to $0$ denoting strong privacy. Intuitively, differential privacy states that 
any event occur with similar probability regardless of the participation
of any individual data to the data set. 
Differential privacy satisfies several properties including 
\emph{composition}, which allows to bound the privacy loss derived by 
multiple applications of DP algorithms to the same dataset, and 
\emph{immunity to post-processing}, which states that the privacy 
loss of DP outputs is not affected by arbitrary data-independent 
post-processing \cite{Dwork:13}.

A function $f$ from a data set $\bm{x} \in \cX$ to a result set 
$R \subseteq \RR^n$ can be made differentially private by injecting 
random noise onto its output. The amount of noise relies on the notion 
of \emph{global sensitivity} %, denoted by $\Delta_f$ and defined as
\(
\Delta_f = \max_{\bm{x} \sim \bm{x}'} \| f(\bm{x}) - f(\bm{x}') \|_1
\), which quantifies the effect of changing an individuals' data
to the output of function $f$. 
The \emph{Laplace mechanism} \cite{dwork:06} that outputs $f(\bm{x}) + 
\bm{\eta}$, where $\bm{\eta} \in \RR^n$ is drawn from the i.i.d.~Laplace 
distribution with $0$ mean and scale $\nicefrac{\Delta_f}{\epsilon}$ over 
$n$ dimensions, achieves $\epsilon$-DP.
% distribution with $0$ mean and scale $\lambda$, denoted by $\Lap(\lambda)$, 
% has a probability density function $\Lap(x|\lambda) = \frac{1}{2\lambda}e^{-\frac{|x|}{\lambda}}$. 
%It can be used to obtain an $\epsilon$-differentially private algorithm to 
%answer numeric queries \cite{Dwork:06}. 

Differential privacy satisfies several important properties. Notably, 
\emph{composability} ensures that a combination of DP mechanisms preserve 
differential privacy.
%%%%%%%%%%%%%%%%%%%%%%%%%%%%%%%%%%%%%%
\begin{theorem}[Sequential Composition]
\label{th:seq_composition}
The composition $(\cM_1(\bm{x}), \ldots, \cM_k(\bm{x}))$ of a collection
$\{\cM_i\}_{i=1}^k$ of $\epsilon_i$-differentially private mechanisms
satisfies $(\epsilon=\sum_{i=1}^{k} \epsilon_i)$-differential privacy.
\end{theorem}
\noindent
The parameter $\epsilon$ resulting in the composition of different
mechanism is referred to as \emph{privacy budget}.  Stronger
composition results exists \cite{kairouz2015composition} but are beyond 
the need of this paper.
%%%%%%%%%%%%%%%%%%%%%%%%%%%%%%%%%%%%%%
\emph{Post-processing immunity} ensures that privacy guarantees are
preserved by arbitrary post-processing steps. 
%%%%%%%%%%%%%%%%%%%%%%%%%%%%%%%%%%%%%%
\begin{theorem}[Post-Processing Immunity] 
\label{th:postprocessing} 
Let $\cM$ be an $\epsilon$-differentially private mechanism and $g$ be
an arbitrary mapping from the set of possible output sequences to an
arbitrary set. Then, $g \circ \cM$ is $\epsilon$-differentially private.
\end{theorem}
%%%%%%%%%%%%%%%%%%%%%%%%%%%%%%%%%%%%%%%%

%%%%%%%%%%%%%%%%%%%%%%%%%%%%%%%%%%%%%%%%%%%%%%%%%%%%%%%%%%%%%%%%%%%%%
\section{Problem Setting and Goals}
\label{sec:setting}
%%%%%%%%%%%%%%%%%%%%%%%%%%%%%%%%%%%%%%%%%%%%%%%%%%%%%%%%%%%%%%%%%%%%%

The paper considers a dataset $\bm{x} \!\in\! \cX \subseteq \RR^k$ of $n$ entities, 
whose elements $x_i= (x_{i1},\ldots, x_{1k})$ describe $k$ measurable 
quantities of entity $i \!\in\! [n]$, such as the number of individuals living 
in a geographical region $i$ and their English proficiency. 
The paper considers two classes of problems: 
\begin{itemize}[leftmargin=*,labelsep=2pt,itemsep=0pt,parsep=2pt,topsep=2pt]
\item An \emph{allotment problem} $P : \cX \times [n] \to \mathbb{R}$ is a function that distributes a finite set of resources to some problem entity. $P$ may represent, for instance, the amount of money allotted to a school district. 
\item A \emph{decision rule} $P: \cX \times [n] \to \{0,1\}$
  determines whether some entity qualifies for some benefits.  For
  instance, $P$ may represent if election ballots should be described
  in a minority language for an electoral district.
% https://www.justice.gov/crt/about-language-minority-voting-rights
\end{itemize}
The paper assumes that $P$ has bounded range, and uses the shorthand
$P_i(\bm{x})$ to denote $P(\bm{x}, i)$ for entity $i$.
The focus of the paper is to study the effects of a DP data-release
mechanism $\cM$ to the outcomes of problem $P$. Mechanism $\cM$ is
applied to the dataset $\bm{x}$ to produce a privacy-preserving
counterpart $\tilde{\bm{x}}$ and the resulting private outcome
$P_i(\tilde{\bm{x}})$ is used to make some allocation decisions.
% \rev{It is worthy noticing that $\cM$ may apply some post-processing 
% step to restrict the randomized output within a feasible region 
% (e.g., to guarantee non-negativity and integrality of the release data).} 
\begin{figure}[!t]
\centering
\includegraphics[width=0.7\columnwidth]{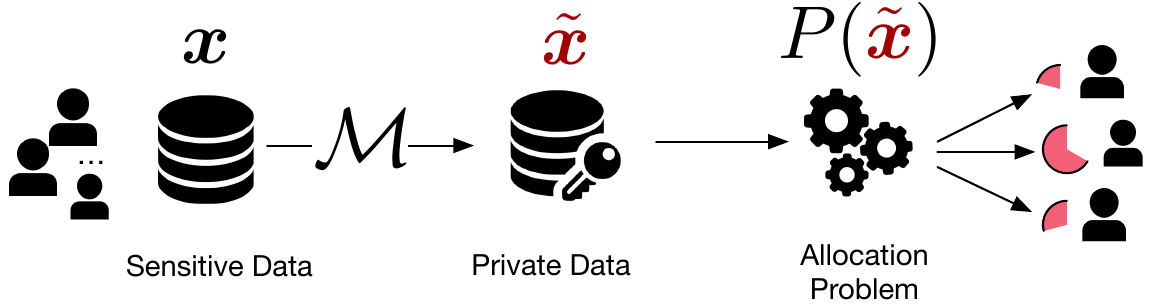}
\caption{Diagram of the private allocation problem.}
\label{fig:framework}
\end{figure}
Figure \ref{fig:framework} provides an illustrative diagram. 

Because random noise is added to the original dataset $\bm{x}$, the
output $P_i(\tilde{\bm{x}})$ incurs some error. {\em The focus of this paper is to characterize and quantify the disparate impact of this
  error among the problem entities}. In particular, the paper focuses
on measuring the bias of problem $P_i$
 \begin{equation}
 \label{eq:bias}
  B_P^i(\cM, \bm{x}) = 
  %\left| 
  \EE_{\tilde{\bm{x}} \sim \cM(\bm{x})} \left[ P_i(\tilde{\bm{x}}) \right] - P_i (\bm{x}),
  %\right|,
 \end{equation}
which characterizes the distance between the expected
privacy-preserving allocation and the one based on the ground truth.
The paper considers the absolute bias $|B_P^i|$, in place of the bias
$B_P^i$, when $P$ is a decision rule. The distinction will become
clear in the next sections.

The results in the paper assume that $\cM$, used to release counts, 
is the Laplace mechanism with an appropriate finite sensitivity $\Delta$.
\emph{However, the results are general and apply to any data-release
  DP mechanism that add unbiased noise}.

%%%%%%%%%%%%%%%%%%%%%%%%%%%%%%%%%%%%%%%%%%%%%%%%%%%%%%%%%%%%%%%%%%%%%
\section{Motivating Problems} % (fold)
\label{sec:motivating_examples}
%%%%%%%%%%%%%%%%%%%%%%%%%%%%%%%%%%%%%%%%%%%%%%%%%%%%%%%%%%%%%%%%%%%%%
This section introduces two Census-motivated problem classes that grant 
benefits or privileges to groups of people. The problems were first 
introduced in \cite{pujol:20}.

\paragraph{Allotment problems}
\newcommand{\tfa}{P^F}
% \emph{Allotment problems} 
% $P\!:\!\cX \!\to\! \RR^n$ 
% distribute a finite set of resources among 
% the problem entities. 

The \emph{Title I of the Elementary and Secondary Education Act of
  1965} \cite{Sonnenberg:16} distributes about \$6.5 billion through
basic grants. The federal allotment is divided among qualifying school
districts in proportion to the count $x_i$ of children aged 5 to 17
who live in necessitous families 
in district $i$. The allocation is formalized by
\begin{align*}
\label{eq:allotment}%\tag{\text{$P_1$}}
  \tfa_i(\bm{x}) \defeq \left( 
    \frac{x_i \cdot a_i}{\sum_{i \in [n] }x_i \cdot a_i}\right),
\end{align*}
where $\bm{x} = (x_i)_{i\in[n]}$ is the vector of all districts counts 
and $a_i$ is a weight factor reflecting students expenditures. 
% \begin{figure}[!tb]
% \centering
% %\includegraphics[width=0.7\linewidth]{P2_Newyork_Bias_Fully_Post_Processed}
% \includegraphics[width=0.7\linewidth]{fig2.pdf}
% \caption{Disproportionate Title 1 Funds Allocation in NY.}
% \label{fig:p1motivation}
% \end{figure}

%%%%%%%%%%%%%%%%%%%%%%%%%%%%%%%%%%%%%%%%%%%%%%%%%%%%%%%%%%%%%%%%%%%%
\begin{wrapfigure}[12]{r}{215pt}
\vspace{-0pt}
\centering
\includegraphics[width=0.97\linewidth]{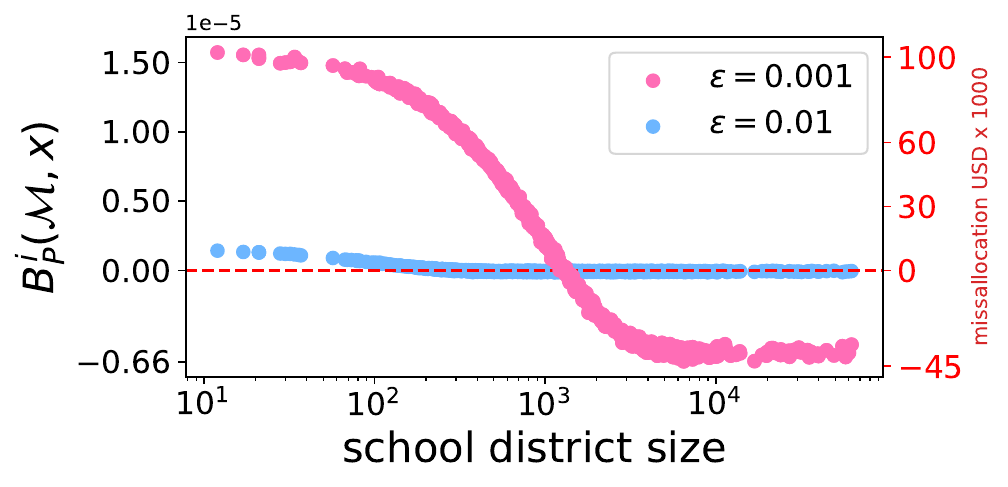}
%\vspace{-12pt}
{\caption{Disproportionate Title 1 funds allotment in NY school 
districts.}\label{fig:p1motivation}}
\end{wrapfigure}
%%%%%%%%%%%%%%%%%%%%%%%%%%%%%%%%%%%%%%%%%%%%%%%%%%%%%%%%%%%%%%%%%%%%%
Figure \ref{fig:p1motivation} illustrates the expected disparity errors arising when using private data as input to problem $\tfa$, for various
privacy losses $\epsilon$. These errors are expressed in terms of bias 
(left y-axis) and USD misallocation (right y-axis) across the different 
New York school districts, ordered by their size. 
The allotments for small districts are typically overestimated while
those for large districts are underestimated. Translated in economic
factors, some school districts may receive up to 42,000 dollars less
than warranted.

\paragraph{Decision Rules}
\emph{Minority language voting right benefits} are granted to qualifying voting jurisdictions. The problem is formalized as
\begin{equation*}
%\tag{$P_{\text{cov}}$}
\label{p:coverage}
P^M_i(\bm{x}) \! \defeq \! \left(
\frac{x_i^{sp}}{x_i^{s}} \!>\! 0.05 \!\lor\! x_i^{sp} \!>\! 10^4
\right) \!\land\! \frac{x_{i}^{spe}}{x_{i}^{sp}} \!>\! 0.0131.
\end{equation*}
%%%%%%%%%%%%%%%%%%%%%%%%%%%%%%%%%%%%%%%%%%%%%%%%%%%%%%%%%%%%%%%%%%%%
\begin{wrapfigure}[13]{r}{215pt}
\vspace{-10pt}
\centering
\includegraphics[width=0.90\linewidth]{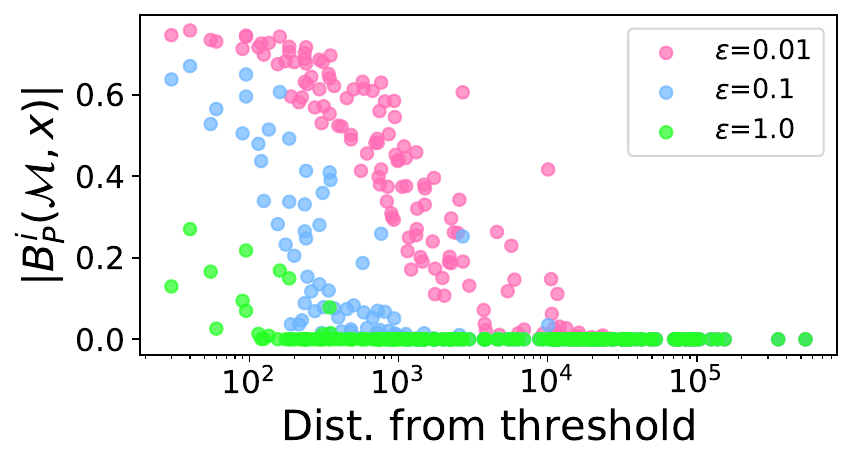}
%\vspace{-12pt}
{\caption{Disproportionate Minority Language Voting Benefits.}\label{fig:p2motivation}}
\end{wrapfigure}
%%%%%%%%%%%%%%%%%%%%%%%%%%%%%%%%%%%%%%%%%%%%%%%%%%%%%%%%%%%%%%%%%%%%%
For a jurisdiction $i$, $x_{i}^{s}$, $x_{i}^{sp}$, and $x_{i}^{spe}$
denote, respectively, the number of people in $i$ speaking the
minority language of interest, those that have also a limited English
proficiency, and those that, in addition, have less than a $5^{th}$
grade education. Jurisdiction $i$ must provide language assistance
(including voter registration and ballots) \emph{iff} $P_i^M(\bm{x})$
is \emph{True}.

Figure \ref{fig:p2motivation} illustrates the decision error (y-axis),
corresponding to the absolute bias $|B^i_{{P^M}}(\cM, \bm{x})|$, for 
sorted $x_i^s$, considering only true positives\footnote{This is 
because misclassification, in this case, implies potentially 
disenfranchising a group of individuals.} for the \emph{Hispanic} language.
The figure shows that there are significant disparities in decision 
errors and that these errors strongly correlate to their distance to 
the thresholds. These issues were also observed in \cite{pujol:20}.

%%%%%%%%%%%%%%%%%%%%%%%%%%%%%%%%%%%%%%%%%%%%%%%%%%%%%%%%%%%%%%%%%%%%%
\section{Fair Allotments and Decision Rules} % (fold)
\label{sec:fair_allocation_and_decision_rules}
%%%%%%%%%%%%%%%%%%%%%%%%%%%%%%%%%%%%%%%%%%%%%%%%%%%%%%%%%%%%%%%%%%%%%
% \subsection{Fairness Metrics}

This section analyzes the fairness impact in allotment problems and
decision rules. The adopted fairness concept captures the desire of
equalizing the allocation errors among entities, which is of paramount
importance given the critical societal and economic impact of the
motivating applications.

\begin{definition}
\label{def:bias}
A data-release mechanism $\cM$ is said fair w.r.t.~a problem $P$ if,
for all datasets $\bm{x} \in \cX$,
\[ 
  	B_P^i(\cM, \bm{x}) = B_P^j(\cM, \bm{x})\quad \forall i,j \in [n].
\]
\end{definition}
\noindent
That is, $P$ does not induce disproportionate errors when taking as
input a DP dataset generated by $\cM$. The paper also introduces a
notion to quantify and bound the mechanism unfairness.
% \nando{We could consider a fairness definition that excludes the specific
% dataset $\bm{x}$ -- the results should hold for any dataset $\bm{x} \in \cX$ to do so.\\
% Proposition 1 is for a particular dataset. This can give issues because we use it 
% to compute $\alpha$ everywhere. Theorem 1, 2, 3, 4 are general. Theorem 5, assumes x<=y, so it is mixed.
% Definition 3 is for a particular dataset.}

\begin{definition}
\label{def:fairness}
A mechanism $\cM$ is said $\alpha$-fair w.r.t.~problem $P$ if, for all
datasets $\bm{x} \in \cX$ and all $i \in [n]$,
\[ 
  \xi_B^i(P, \mathcal{M}, \bm{x}) = \max_{j \in [n]}
  \left|B_P^i(\cM, \bm{x}) - B_P^j(\cM, \bm{x})\right|  \leq \alpha,
\]
where $\xi_B^i$ is referred to as the \emph{disparity error} of entity $i$. 
\end{definition}
\noindent
Parameter $\alpha$ is called the \emph{fairness bound} and captures
the fairness violation, with values close to $0$ denoting strong
fairness. A fair mechanism is also $0$-fair.
% \begin{proposition}
% Let $\cM$ be an $\alpha$-unfair mechanism w.r.t.~a problem $P$ and data 
% set $\bm{x}$. Then, $\cM$ is also ${\alpha}^\uparrow$-unfair, with 
% ${\alpha}^\uparrow = \max_i B_P^i(\cM, \bm{x}) - \min_i B_P^i(\cM, \bm{x})$.
% \end{proposition}
% \nando{Check this again -- why are they not the same?}
% \noindent
% Value $\alpha^\uparrow$ is denoted \emph{fairness bound}. 
% %and it is tight (i.e., ${\alpha}^\uparrow=\alpha$) when $\min_{i \in [n]} B_P^i(\cM, \bm{x}) = 0$. 
% While this bound is useful for the development of some of the theoretical 
% results discussed in the paper, the paper also found that it is often 
% tight for the experiments considered that use real census data.

Note that computing the fairness bound $\alpha$ analytically may not be
feasible for some problem classes, since it may involve computing the 
expectation of complex functions $P$. Therefore, in the analytical 
assessments, the paper recurs to a sampling approach to compute 
the \emph{empirical expectation} $\hat{E}[P_i(\tilde{\bm{x}})] = \frac{1}{m} \sum_{j\in[m]}P_i(\tilde{\bm{x}}^j)$
in place of the true expectation in Equation~\eqref{eq:bias}. 
Therein, $m$ is a sufficiently large sample size and $\tilde{\bm{x}}^j$ 
is the $j$-th outcome of the application of mechanism $\cM$ on data set $\bm{x}$.
\subsection{Fair Allotments: Characterization}
%%%%%%%%%%%%%%%%%%%%%%%%%%%%%%%%%%%%%%%%%%%%%%%%%%%%%%%%%%%%%%%%%%%%%

The first result characterizes a sufficient condition for the
allotment problems to achieve finite fairness violations.
The presentation uses $\bm{H}P_i$ to denote the Hessian of problem $P_i$, 
%with entries $(\bm{H}P_i)_{ab} = \frac{\partial^2 P_i}{\partial x_a \partial x_b}$.
and $\Tr(\cdot)$ to denote the trace of a matrix. In this context, the
Hessian entries are functions receiving a dataset as input. The
presentation thus uses 
$(\bm{H}P_i)_{j,l}(\bm{x})$ and
$\Tr(\bm{H}P_i)(\bm{x})$ to denote the
application of the second partial derivatives of $P_i$ and of 
the \emph{Hessian trace function} on dataset $\bm{x}$.
    
% \begin{theorem}
% \label{lem:fair_bound_allottments}
% Let $P$ be an allotment problem which is at least twice differentiable.
% A data-release mechanism $\cM$ is $\alpha$-fair w.r.t.~$P$ for some 
% $\alpha < \infty$ if there exist some constant values 
% $c^i_{jl} \; (i \in [n], j,l \in [k])$ such that, for all datasets $\bm{x} \in \cX$, 
% \[
%   (\bm{H}P_i)_{j,l}(\bm{x}) = c^i_{j,l}   \;\; (i\in[n]\; j,l\in[k]).
% \]
% \end{theorem}

\begin{theorem}
\label{lem:fair_bound_allottments}
Let $P$ be an allotment problem that is at least twice differentiable.
A data-release mechanism $\cM$ is $\alpha$-fair w.r.t.~$P$, for some 
finite $\alpha$, if for all datasets $\bm{x} \in \cX$ 
the entries of the Hessian $\bm{H}P_{i}$ of problem $P_i$ are a constant 
function, that is, if there exists $c^i_{jl} \in \mathbb{R} \; (i \in [n], j,l \in [k])$ such that, 
\begin{equation}
\label{eqn:thm3}
  (\bm{H}P_i)_{j,l}(\bm{x}) = c^i_{j,l}   \;\; (i\in[n]\; j,l\in[k]).
\end{equation}
\end{theorem}

\begin{proof}
Firstly, notice that the problem bias (Equation \ref{eq:bias}) can be expressed as
\begin{subequations}
\begin{align}
  B^i_P(\cM, \bm{x}) &= \EE[P_i(\tilde{\bm{x}} = \bm{x} + \eta)] - P_i(\bm{x}) \\
  \label{eq:p1}
  % linearity of ex. and taylor
      &\approx P_i(\bm{x}) 
      + \EE\left[ \eta \nabla P_i(\bm{x})\right]
      + \EE\left[ \frac{1}{2} \eta^T \bm{H}P_i(\bm{x}) \eta\right] - P_i(\bm{x}) \\
  % \label{eq:p2}
  % % independence of n and \nabla P_i and since E[\eta = 0]
  %               &= \EE\left[\eta\right] \EE\left[\nabla P_i(\bm{x})\right] 
  %                 + \EE\left[\frac{1}{2} \eta^T \bm{H}P_i(\bm{x}) \eta\right]
 \label{eq:p2}
  % independence of n and \nabla P_i and since E[\eta = 0]
    &= \EE\left[\frac{1}{2} \eta^T \bm{H}P_i(\bm{x}) \eta\right]\\
 \label{eq:p3}
  % by definition of Hessian
    &= \frac{1}{2} \EE\left[ 
    \sum_{j,k \in [n]} \eta_j (\bm{H}P_i)_{jk}(\bm{x}) 
    \eta_k\right]\\
 \label{eq:p4}
  % since \eta is a vect of independent noise and thus E[n_k n_j] =0
    &= \frac{1}{2} \EE\left[ 
    \sum_{j \in [n]} \eta_j^2 (\bm{H}P_i)_{jj}(\bm{x}) \right]\\
 \label{eq:p5}
    &= \frac{1}{2} \sum_{j \in [n]} \EE\left[ \eta_j^2 \right] 
    \sum_{j \in [n]} \EE\left[ (\bm{H}P_i)_{jj}(\bm{x}) \right]\\
 \label{eq:p6}
    &= \frac{1}{2} n \var[\eta] \, \Tr\left(\bm{H}P_i\right)(\bm{x}),
 % \label{eq:p7}
 %  &= \frac{n \Delta^2}{\epsilon^2} \Tr\left(\bm{H}P_i\right)(\bm{x})
\end{align}
\end{subequations}
where the approximation (in \eqref{eq:p1}) uses a Taylor expansion of
the private allotment problem $P_i(\bm{x} + \eta)$, where $\eta =
\Lap(\nicefrac{\Delta}{\epsilon})$ and the linearity of
expectations. 
Equation \eqref{eq:p2} follows from independence of $\eta$ and $\nabla P_i(\bm{x})$ and from the assumption of unbiased noise (i.e., $\EE[\eta] = 0$) and \eqref{eq:p4} from independence of the elements of  $\eta$ and thus $\EE[n_k n_j] =0$ for $j \neq k$.
Finally, \eqref{eq:p6} follows from $\EE[\eta^2] = \var[\eta] + (\EE[\eta])^2$ and $\EE[\eta] = 0$ again, and where $\Tr$ denotes the trace of the Hessian matrix.

The bias $B^i_P$ can thus be approximated by an expression
involving the local curvature of the problem $P_i$ and the variance of the noisy input. 

Next, by definition of bounded fairness \ref{def:fairness}
\begin{subequations}
\begin{align}
  \xi_B^i(P, \cM, \bm{x}) =
  \label{eq:p1_5a}
  &\max_{j \in [n]} \left|B_P^i(\cM, \bm{x}) - B_P^j(\cM, \bm{x}) \right| \leq \alpha \\
  \Leftrightarrow
  & 
  \label{eq:p1_5b}
  n \var[\eta] \left| 
    \Tr(\bm{H}P_i)(\bm{x}) - \Tr(\bm{H}P_j)(\bm{x})  \right| \leq 
    \alpha\;\;\forall j \in [n].
\end{align}
\end{subequations}
Since, by assumption, there exists constants $c_k$ such that $\forall x \in \cX$, \(\Tr(\bm{H}P_k)(\bm{x}) =  \sum_{j,l} c^{k}_{j,l} = c_k\) for $k\in [n]$, it follows, that
\begin{align*}
 n \var[\eta] \left| c_i - c_j \right| < n \var[\eta] 
 \left( \max_{i \in [n]} c_i - \min_{i \in [n]} c_i \right)  < \infty.
\end{align*}
\end{proof}

The above shed light on the relationship between fairness and the
difference in the local curvatures of problem $P$ on any pairs of
entities. As long as this local curvature is constant across all
entities, then the difference in the bias induced by the noise onto
the decision problem of any two entities can be bounded, and so can
the (loss of) fairness. 

An important corollary of Theorem \ref{lem:fair_bound_allottments} illustrates which restrictions on the
structure of problem $P$ are needed to satisfy fairness.
\begin{corollary}
\label{cor:1}
If $P$ is a linear function, then $\cM$ is fair w.r.t.~$P$.
\end{corollary}
\begin{proof}
The result follows by noticing that the second derivative of linear function is $0$ for any input. Thus, for any $i \in [n]$, and $\bm{x} \in \cX, \Tr(\bm{H}P_i)(\bm{x}) = 0$.
Therefore, from \eqref{eq:p1_5b}, for every $i \in [n]$,
\begin{align*}
\xi_B^i(P, \cM, \bm{x}) &= \max_{j \in [n]} 
\left|\Tr(\bm{H}P_i)(\bm{x}) - \Tr(\bm{H}P_j)(\bm{x}) \right| = 0.
\end{align*}
\end{proof}

A more general result is the following.
\begin{corollary}
\label{cor:2}
$\cM$ is fair w.r.t.~$P$ if there exists a constant $c$ such that,
for all dataset $\bm{x}$, 
\[
\Tr(\bm{H}P_i)(\bm{x}) = c \;\; (i \in [n]).
\]
\end{corollary}
The proof is similar, in spirit, to proof of Corollary \ref{cor:1}, noting that, in the above, the constant $c$ is equal among all Traces of the Hessian of problems $P_i$ $(i \in [n])$.

%%%%%%%%%%%%%%%%%%%%%%%%%%%%%%%%%%%%%%%%%%%%%%%%%%%%%%%%%%%%%%%%%%%%%
\subsection{Fair Decision Rules: Characterization}
%%%%%%%%%%%%%%%%%%%%%%%%%%%%%%%%%%%%%%%%%%%%%%%%%%%%%%%%%%%%%%%%%%%%%

The next results bound the fairness violations of a class of indicator
functions, called \emph{thresholding functions}, and discusses the
loss of fairness caused by the \emph{composition of boolean
  predicates}, two recurrent features in decision rules.  The fairness
definition adopted uses the concept of absolute bias, in place of bias
in Definition \ref{def:bias}.  Indeed, the absolute bias $|B^i_{P}|$
corresponds to the classification error for (binary) decision rules of
$P_i$, i.e., $\Pr[ P_i(\tilde{\bm{x}}) \neq P_i(\bm{x})]$. The results
also assume $\cM$ to be a non-trivial mechanism, i.e., $| B^i_{P}(\cM,
\bm{x}) | < 0.5 \, \forall i \in [n]$.  Note that this is a
non-restrictive condition, since the focus of data-release mechanisms
is to preserve the quality of the original inputs, and the mechanisms
considered in this paper (and in the DP-literature, in general) all
satisfy this assumption.

\begin{theorem}
\label{thm:bias_theshold}
Consider a decision rule  $P_i(\bm{x}) \!=\! \mathbbm{1}\{ x_i \!\geq\ \ell \}$ 
for some real value $\ell$. Then, mechanism $\cM$ is $0.5$-fair 
w.r.t.~$P_i$.
\end{theorem}

\begin{proof}
From Definition \ref{def:fairness} (using the absolute bias 
$|B^i_P(\cM, \bm{x})|$), and since the absolute bias is always 
non-negative, it follows that, for every $i \in [n]$:
\begin{subequations}
\label{eq:p_6}
\begin{align}
     \xi_B^i(P, \mathcal{M}, \bm{x}) &= \max_{j \in [n]}
      \left| |B_P^i(\cM, \bm{x})| - |B_P^j(\cM, \bm{x})|\right|  \\
      & \leq \max_{j \in [n]} |B_P^j(\cM, \bm{x})| - \min_{j \in [n]} |B_P^j(\cM, \bm{x})| \\
      & \leq  \max_{j \in [n]} |B_P^j(\cM, \bm{x})|.
\end{align}
\end{subequations}
Thus, by definition, mechanism $\cM$ is $\max_{j \in [n]} |B^P_j(\cM, \bm{x})|$-fair 
w.r.t.~problem $P$.
The following shows that the maximum absolute bias 
$\max_{j \in [n]} |B^P_j(\cM, \bm{x})| \leq 0.5$.
%Let $i$ be the entry with the largest absolute bias.
W.l.o.g.~consider an entry $i$ and the case in which $P_i(\bm{x}) = 
\textsl{True}$ (the other case is symmetric). It follows that, 
\begin{subequations}
\label{eq:p_7}
\begin{align}
\label{eq:p_7a}
        |B^i_P(\cM, \bm{x})| & = | P_i(\bm{x}) - \mathbb{E}_{\tilde{\bm{x}}_i \sim \cM(\bm{x})} [P_i(\tilde{\bm{x}})]  | \\
\label{eq:p_7b}
        & = | 1 - \Pr( \tilde{x}_i \geq \ell )| \\
\label{eq:p_7c}
        & = |1 - \Pr( \eta \geq \ell - x_i)|,
\end{align}
\end{subequations}
where $\eta \sim \mbox{Lap}(0, \nicefrac{\Delta}{\epsilon})$. 
Notice that,
\begin{equation} 
\label{eq:p_8}
\Pr( \eta  \geq \ell - x_i) \geq \Pr( \eta \geq 0) = 0.5,
\end{equation}
since $\ell-x_i \leq 0$, by case assumption (i.e., $P_i(\bm{x}) = \textsl{True}$ implies
that $x_i \geq \ell$) and by that the mechanism considered adds 
0-mean symmetric noise.  Thus, from \eqref{eq:p_7c} and \eqref{eq:p_8}, 
$|B^{i}_P(\cM, \bm{x})| \leq 0.5$, and since, the above holds for any
entity $i$, it follows that 
\begin{equation}
\label{eq:p_9}
\max_{j \in [n]} |B_P^j(\cM, \bm{x})| \leq 0.5
\end{equation}
and thus, for every $i \in [n]$, $\xi_B^i(P, \mathcal{M}, \bm{x}) \leq 0.5$,
and, therefore, from \eqref{eq:p_6} and \eqref{eq:p_9}, $\cM$ is $0.5$-fair.
\end{proof}

\noindent This is a worst-case result and the mechanism may enjoy a
better bound for specific datasets and decision rules. It is however
significant since thresholding functions are ubiquitous in decision
making over census data.

\smallskip
The next results focus on the composition of Boolean predicates
under logical operators. The results are given under the assumption
that mechanism $\cM$ adds independent noise to the inputs of the
predicates $P_1$ and $P_2$ to be composed, which is often the
case. This assumption for $P_1$ and $P_2$ is denoted by
$P^1 \independent P^2$. %Future work will aim at generalizing this results to broader assumptions.

The paper first introduces the following properties and Lemmas whose 
proofs are reported in the appendix. 
\begin{property}
\label{lem:increasing_func}
The following three bivariate functions: 
$f(a,b) = ab$, 
$f(a,b) = a+ b - ab$, and 
$f(a,b) = a+b -2ab$, with support $[0, 0.5]$ and range 
$\mathcal{R}$ all are monotonically increasing on their support. 
\end{property}

\renewcommand{\Pr}[1]{\text{Pr}\left(#1\right)}
\newcommand{\False}{\textsl{False}}
\newcommand{\True}{\textsl{True}}

\begin{lemma}
  \label{lem:and_formula}
  Consider predicates $P_i^1$ and $P_i^2$ and let $P_i = P_i^1 \land P_i^2$, then,
  for any dataset $\bx \in \cX$, 
\begin{enumerate}[leftmargin=*,labelsep=2pt,itemsep=0pt,parsep=2pt,topsep=2pt,label=(\roman*)]
\item $P_i^1(\bx) =  \False \land P_i^2(\bx) =  \False \; \Rightarrow
      \Pr{ P_i(\tilde{\bx}) \neq P_i(\bx) } = |B^i_{P_i^1}| |B^i_{P_i^2}| $
\label{lemma1:c1}
\item  $P_i^1(\bx) = \False \land P_i^2(\bx) = \True \Rightarrow 
       \Pr{ P_i(\tilde{\bx}) \neq P_i(\bx)} = |B^i_{P_i^1}| (1-|B^i_{P_i^2}|) $
\label{lemma1:c2}
\item $P_i^1(\bx) = \True \land P_i^2(\bx) =  \False \Rightarrow 
       \Pr{P_i(\tilde{\bx}) \neq P_i(\bx)} = (1-|B^i_{P_i^1}|) |B^i_{P_i^2}| $
\label{lemma1:c3}
\item  $P_i^1(\bx) = \True \land P_i^2(\bx) = \True \Rightarrow
       \Pr{P_i(\tilde{\bx}) \neq P_i(\bx)} = |B^i_{P^1}| + |B^i_{P^2}| - |B^i_{P^1}| |B^i_{P^2}| $,
\label{lemma1:c4}
\end{enumerate}
where $\tilde{\bm{x}} = \cM(\bx)$ is the privacy-preserving dataset. 
\end{lemma}

\begin{lemma}
  \label{lem:or_formula}
  Consider predicates $P_i^1$ and $P_i^2$ and let $P_i = P_i^1 \lor P_i^2$, then,
  for any dataset $\bx \in \cX$,
\begin{enumerate}[leftmargin=*,labelsep=2pt,itemsep=0pt,parsep=2pt,topsep=2pt,label=(\roman*)]
\item $P^1_i(\bx) = \False, P^2_i(\bx) = \False\; \Rightarrow 
  \Pr{P_i(\tilde{\bx}) \neq P_i(\bx)} = |B^i_{P^1}| + |B^i_{P^2}| - |B^i_{P^1}| |B^i_{P^2}| $
\item  $P^1_i(\bx) = \False, P^2_i(\bx) = \True\; \Rightarrow 
  \Pr{P_i(\tilde{\bx}) \neq P_i(\bx)} = (1-|B^i_{P^1}|) |B^i_{P^2}|  $
\item $P^1_i(\bx) = \True, P^2_i(\bx) = \False\; \Rightarrow 
  \Pr{P_i(\tilde{\bx}) \neq P_i(\bx)} =  |B^i_{P^1}| (1-|B^i_{P^2}|)$
\item  $P^1_i(\bx) = \True, P^2_i(\bx) = \True\; \Rightarrow 
  \Pr{P_i(\tilde{\bx}) \neq P_i(\bx} =  |B^i_{P^1}| |B^i_{P^2}|  $,
\end{enumerate}
where $\tilde{\bm{x}} = \cM(\bx)$ is the privacy-preserving dataset. 
\end{lemma}

\begin{lemma}
\label{lem:xor_formula}
Given $P(\bx) = P^1(\bx) \oplus P^2(\bx)$, then 
for any value of  $P^1_i(\bx), P^2_i(\bx) \in \{\False, \True\}$:
$$ \Pr{P_i(\tilde{\bx}) \neq P_i(\bx)}  
= |B^i_{P^1}| + |B^i_{P^2}| - 2|B^i_{P^1}| |B^i_{P^2}|. $$
\end{lemma}

\begin{theorem}
\label{thm:and_unfairness}
Consider predicates $P^1$ and $P^2$ such that $P^1 \!\independent\! P^2$
and assume that mechanism $\cM$ is $\alpha_k$-fair for
predicate $P^k$ $(k \in \{1,2\}$). Then $\cM$ is $\alpha$-fair for 
predicates $P^1 \lor P^2$ and
$P^1 \land P^2$ with 
\begin{equation}
\label{thm5:eq}
\alpha \!=\! \left(
  \alpha_1 \!+\! \underline{B}^1 \!+\! \alpha_2 \!+\! \underline{B}^2 
\!-\! (\alpha_1 \!+\! \underline{B}^1) (\alpha_2 \!+\! \underline{B}^2) 
\!-\! \underline{B}^1 \underline{B}^2
\right),
\end{equation}
where $\overline{B}^k$ and $\underline{B}^k$ are the maximum 
and minimum absolute biases for $\cM$ w.r.t.~$P^k$ (for $k=\{1,2\}$).
\end{theorem}

\begin{proof}
The proof focuses on the case $P^1 \land P^2$ 
while the proof for the disjunction is similar.

First notice that by Lemma \ref{lem:and_formula} and assumption 
of $\cM$ being non-trivial, it follows that
\begin{align}
  |B^i_{P^1}||B^i_{P^2}| &< |B^i_{P^1}|(1-|B^i_{P^2}|), \\ 
  |B^i_{P^2}|(1-|B^i_{P^1}|) &< 
  |B^i_{P^1}| + |B^i_{P^2}| - |B^i_{P^1}| 
  |B^i_{P^2}|.
\end{align}
due to that $0 \leq |B^i_{P^1}| \leq 0.5$
and $0 \leq |B^i_{P^2}| \leq 0.5$, and thus:
\begin{subequations}
\begin{align}
    |B^i_{P^1}| |B^i_{P^2}| &\leq \Pr{P_i(\tilde{\bx})\neq  P_i(\bx)} \\
                            &\leq |B^i_{P^1}| +|B^i_{P^2}| - |B^i_{P^1}| |B^i_{P^2}|,
\end{align}
\end{subequations}
From the above, the maximum absolute bias $\overline{B}_P$ 
can be upper bounded as:
\begin{subequations}
\begin{align}
    \overline{B}_P 
    & = \max_{i} \Pr{P_i(\tilde{x}) \neq P_i(x)} \\
    & \leq \max_i |B^i_{P^1}| +|B^i_{P^2}| - |B^i_{P^1}| |B^i_{P^2}| \\
    & = \overline{B^1} + \overline{B^2} -\overline{B^1}\overline{B^2},
\end{align}
\end{subequations}
where the first inequality follows by Lemma 
\ref{lem:and_formula} and the last equality 
follows by Property \ref{lem:increasing_func}.

Similarly, the minimum absolute bias of $\underline{B}_P$
can be lower bounded as:
\begin{subequations}
\begin{align}
    \underline{B}_P 
     & =  \min_{i} \Pr{P_i(\tilde{x}) \neq P_i(x)} \\
     & \geq \min_{i} |B^i_{P^1}| |B^i_{P^2}|   = \underline{B^1} \underline{B^2},
\end{align}
\end{subequations}

where the first inequality is due to Lemma \ref{lem:and_formula}, and the last equality is due to Property \ref{lem:increasing_func}. Hence, the level of unfairness $\alpha$ of problem $P$ can be determined by:
\begin{align}
\begin{split}
\label{eq:alpha_bound}
    \alpha =\overline{B}_P - \underline{B}_P \leq \overline{B^1} + \overline{B^2} -\overline{B^1}\overline{B^2} - \underline{B^1} \underline{B^2}.
    \end{split}
\end{align}

Substituting 
$\overline{B^1} = (\alpha_1 + \underline{B^1})$ 
and $\overline{B^2} = (\alpha_2 + \underline{B^2})$ 
into Equation \eqref{eq:alpha_bound}  gives the sought fairness bound.
\end{proof}
\noindent
The result above bounds the fairness violation derived by the 
composition of Boolean predicates under logical operators.

\begin{theorem}
\label{thm:xor_unfairness}
Consider predicates $P^1$ and $P^2$ such that $P^1 \!\independent\! P^2$
and assume that mechanism $\cM$ that is $\alpha_k$-fair for
predicate $P^k$ $(k \!\in\! \{1,2\}$). Then $\cM$ is $\alpha$-fair for $P^1 \oplus P^2$ 
with 
\begin{equation}
\label{thm6:eq}
\alpha \!=\! \left( 
     \alpha_1 (1-2\underline{B}^2) + \alpha_2 (1- 2\underline{B}^1) 
  - 2\alpha_1 \alpha_2 
  \right),
\end{equation}
where $\underline{B}^k$ is the minimum absolute bias for $\cM$
w.r.t.~$P^k$ ($k\!=\!\{1,2\}$).
\end{theorem}

\begin{proof}
First, notice that the maximum absolute bias for $\cM$ w.r.t.~$P = P^1 \oplus P^2$ 
can be expressed as: 
\begin{subequations}
\begin{align}
      \max_{i} \Pr{P_i(\tilde{\bx}) \neq P_i(\bx)} 
    &=\max_{|B^i_{P^1}|, |B^i_{P^2}|} |B^i_{P^1}| + |B^i_{P^2}|
     - 2|B^i_{P^1}| |B^i_{P^2}| \\
    &= \overline{B}^1 + \overline{B}^2 - 2\overline{B}^1\overline{B}^2,
\end{align}
\end{subequations}
\noindent
where the first equality is due to Lemma 
\ref{lem:xor_formula}, and the second due to Property 
\ref{lem:increasing_func}. 

Similarly, the minimum absolute bias for $\cM$ w.r.t.~ $P = P^1 \oplus P^2$ 
can be expressed as:
\begin{subequations}
\begin{align}
  \min_{i} \Pr{P_i(\tilde{x})\neq  P_i(x)}
  &= \min_{|B^i_{P^1}|, |B^i_{P^2}|} |B^i_{P^1}| + |B^i_{P^2}| 
     - 2|B^i_{P^1}| |B^i_{P^2}| \\
  &= \underline{B}^1 + \underline{B}^2 - 2\underline{B}^1\underline{B}^2.
\end{align}
\end{subequations}
Since the fairness bound $\alpha$ is defined as the difference 
between the maximum and the minimum absolute biases, it follows:
\begin{subequations}
\begin{align}
  \alpha &= \max_{i} \Pr{P_i(\tilde{x}) \neq  P_i(x)} - \min_{i} \Pr{P_i(\tilde{x})  \neq  P_i(x)}  \noindent\\
         &= \overline{B}^1_i + \overline{B}^2 - 2\overline{B}^1_i \overline{B}^2 
          - \underline{B}^1_i + \underline{B}^2 - 2\underline{B}^1_i \underline{B}^2,
\end{align}
\end{subequations}
Replacing 
$\overline{B}^1_i = \underline{B}^1_i +\alpha_1$ and  
$\overline{B}^2 = \underline{B}^2 +\alpha_2$, 
gives the sought fairness bound.
\end{proof}

\noindent
The following is a direct consequence of Theorem~\ref{thm:xor_unfairness}.

\begin{corollary}
Assume that mechanism $\cM$ is fair w.r.t.~problems $P^1$ 
and $P^2$. Then $\cM$ is also fair w.r.t.~$P^1\oplus P^2$.
\end{corollary}
The above is a direct consequence of Theorem \ref{thm:xor_unfairness}
for $\alpha_1 =0, \alpha_2 =0$. 
While the XOR operator $\oplus$ is not adopted in the case studies considered
in this paper, it captures a surprising, positive compositional fairness 
result.

%%%%%%%%%%%%%%%%%%%%%%%%%%%%%%%%%%%%%%%%%%%%%%%%%%%%%%%%%%%%%%%%%%%%%
\section{The Nature of Bias}
%%%%%%%%%%%%%%%%%%%%%%%%%%%%%%%%%%%%%%%%%%%%%%%%%%%%%%%%%%%%%%%%%%%%%

The previous section characterized conditions bounding fairness
violations.  In contrast, this section analyzes the reasons for
disparity errors arising in the motivating problems.

\iffalse
\rev{This section analyzes the reasons for unfairness outlined in the results 
presented in the previous section. The results presented below are 
general and capture a broader classes of problems than those outlined 
in the previous section.
Additionally, while the results presented in this section focus on the
Laplace mechanism, they apply to any data-release mechanism based 
that uses \emph{unbiased noise}, including the Geometric and the 
Gaussian mechanisms. 
% and assume that the allocation problem is expressed by an at least twice differentiable function. 
}
\fi

\subsection{The Problem Structure}
The first result is an important corollary of Theorem \ref{lem:fair_bound_allottments}. 
% It remarks that any problem $P$ whose Hessian is not constant across the 
% entities (i.e., $HP_i \neq HP_j$ for $i\neq j$) cannot achieve perfect fairness. 
% \emph{This includes any non-convex function}, as is the case for \emph{all} 
% the allocation problems considered in this paper. 
It studies which restrictions on the structure of problem $P$ are
needed to satisfy fairness. Once again, $P$ is assumed to be at least
twice differentiable.

\begin{corollary}
\label{cor:4}
Consider an allocation problem $P$. Mechanism $\cM$ is not fair
w.r.t.~$P$ if there exist two entries $i, j \in [n]$ such that
$\Tr(\bm{H}P_i)(\bm{x}) \neq \Tr(\bm{H}P_j)(\bm{x})$ for some dataset
$\bm{x}$.
\end{corollary}

\noindent
The corollary is a direct consequence of Theorem \ref{lem:fair_bound_allottments}. %-- see Equation \eqref{eq:p1_12}.
It implies that fairness cannot be achieved if $P$ is \emph{a
  non-convex function}, as is the case for \emph{all} the allocation
problems considered in this paper. {\em A fundamental consequence of
  this result is the recognition that adding Laplacian noise to the
  inputs of the motivating example will necessarily introduce fairness
  issues.} 

\begin{example}
For instance, consider $\tfa$ and notice that the trace of
its Hessian
\[
\Tr(\bm{H}\tfa_i) = 2a_i \left[
\frac{x_i \sum_{j\in[n]} a_j^2 - a_i \left(\sum_{j\in[n]} x_ja_j\right)}
     {\left(\sum_{j\in[n]}x_ja_j\right)^3}
\right],
\]
is not constant with respect to its inputs. Thus, any two entries $i,
j$ whose $x_i \neq x_j$ imply $\Tr(\bm{H}\tfa_i) \neq
\Tr(\bm{H}\tfa_j)$.  As illustrated in Figure \ref{fig:p1motivation},
Problem $\tfa$ can introduce significant disparity errors.  For $\epsilon = 0.001,
0.01$, and $0.1$ the estimated fairness bounds are $0.003$, $3\times
10^{-5}$, and $1.2\times 10^{-6}$ respectively, which amount to an
average misallocation of \$43,281, \$4,328, and \$865.6 respectively.
The estimated fairness bounds were obtained by performing a linear
search over all $n$ school districts and selecting the maximal
$\Tr(\bm{H}\tfa_i)$.
\end{example}

\begin{figure*}[!t]
\centering
\includegraphics[height=70pt]{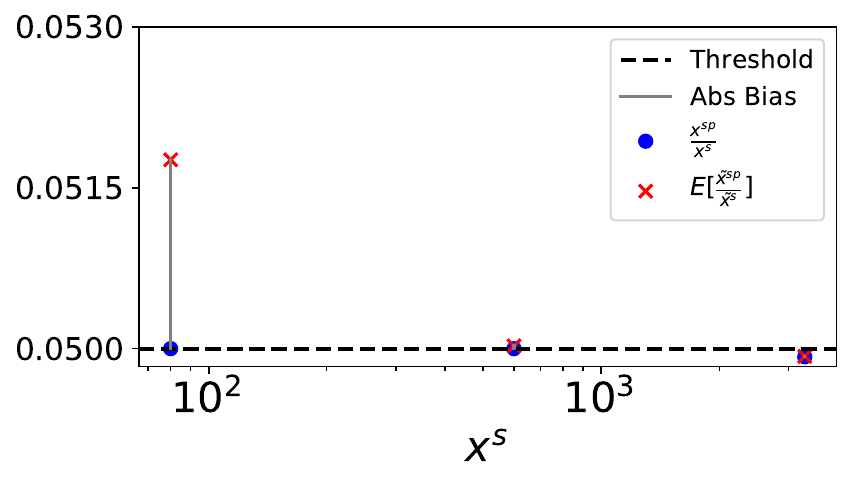}
\includegraphics[height=70pt]{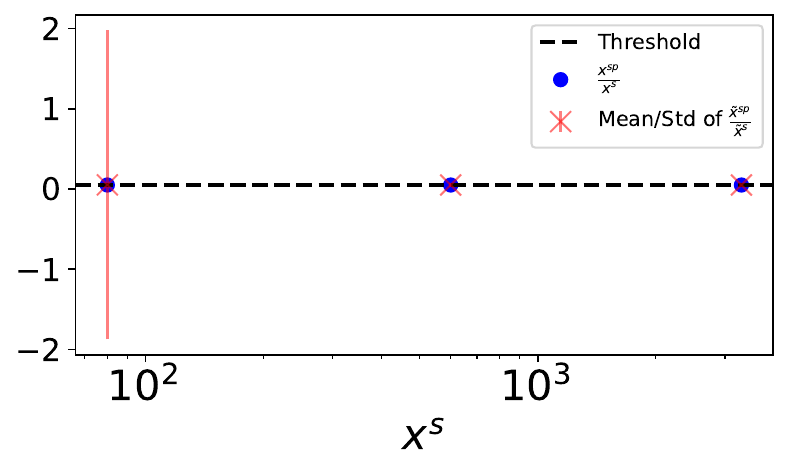}
\includegraphics[height=70pt]{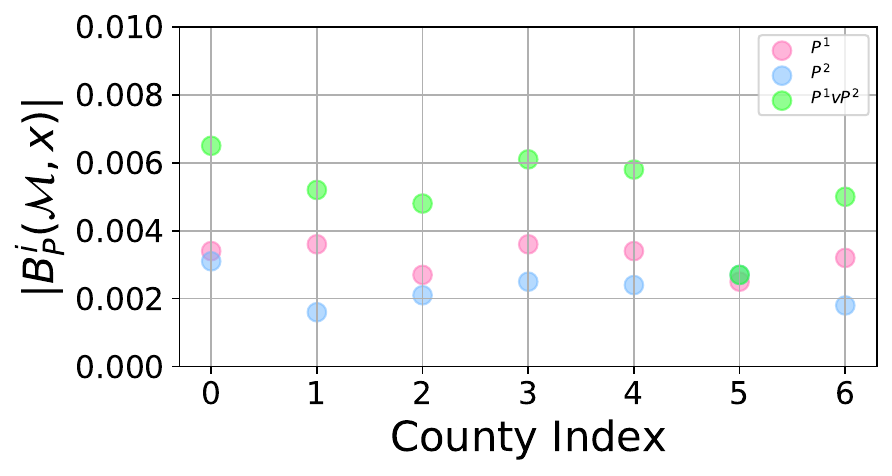}
\caption{Unfairness effect in \emph{ratios} (left), \emph{thresholding} (middle)
and predicates disjunction (right)}
\label{fig:impact_shape}
\end{figure*}

\paragraph{Ratio Functions}
The next result considers \emph{ratio functions} of the form
$P_i(\langle x,y \rangle) \!=\! \nicefrac{x}{y}$ with $x,y \!\in\! \mathbb{R}$
and $x \!\leq\! y$, which occur in the Minority language voting right
benefits problem $P_i^M$. In the following $\cM$ is the Laplace mechanism. 
% \begin{theorem}
% Mechanism $\cM$ is $\alpha$-fair w.r.t.~$P_i(\langle x,y \rangle) 
% = \nicefrac{x}{y}$ and inputs $x, y$, with $\alpha \leq {\frac{2\Delta^2}{\epsilon^2}}$.
% \end{theorem}
\begin{corollary}
Mechanism $\cM$ is not fair w.r.t.~$P_i(\langle x,y \rangle) = \nicefrac{x}{y}$ and inputs $x, y$.
\end{corollary}
\noindent

The above is a direct consequence of Corollary \ref{cor:4}.

% \rev{The proof relies on Theorem \ref{lem:fair_bound_allottments} and on
% the fact that the ratio is bounded by $1$.}  
Figure
\ref{fig:impact_shape} (left) provides an illustration linked
to problem $P^M$.  It shows the original values
$\nicefrac{x^{sp}}{x^{s}}$ (blue circles) and the expected values of
the privacy-preserving counterparts (red crosses) of three counties;
from left to right: \emph{Loving county, TX}, where
$\nicefrac{x^{sp}}{x^{s}} \!=\! \nicefrac{4}{80} \!=\! 0.05$,
\emph{Terrell county, TX}, where $\nicefrac{x^{sp}}{x^{s}} \!=\!
\nicefrac{30}{600} \!=\! 0.05$, and \emph{Union county, NM}, where
$\nicefrac{x^{sp}}{x^{s}} \!=\! \nicefrac{160}{3305} \!=\! 0.0484$.
The length of the gray vertical line represents the absolute bias and
the dotted line marks a threshold value ($0.05$) associated with the
formula $P^M_i$.  While the three counties have (almost) identical
ratios values, they induce significant differences in absolute
bias. This is due to the difference in scale of the numerator (and
denominator), with smaller numerators inducing higher bias.

%%%%%%%%%%%%%%%%%%%%%%%%%%%
%% APPENDIX
%%%%%%%%%%%%%%%%%%%%%%%%%%%
% \begin{example}
% Consider Problem $P_3$, \emph{Apportionment  of  Legislative  Representative}, where $x_i = x^s_i$ the population of state $i$, and one intermediate output $f(x_i) = \frac{x_i}{\sum_{j \in [n]} x_j}$. 
%  We have $\mbox{trace}[H]_{x_i} = \frac{2 \big(n x^s_i - \sum_{j \in [n]} x^s_j \big)}{ (\sum_{j \in [n]} x^s_j)^3}$ which also is not constant w.r.t $x^s_i$. Hence the main mapping function in Problem $P_3$ will introduce unfairness.
% \end{example}
%%%%%%%%%%%%%%%%%%%%%%%%%%%

\paragraph{Thresholding Functions}

As discussed in Theorem \ref{thm:bias_theshold}, discontinuities
caused by indicator functions, including thresholding, may induce
unfairness. This is showcased in Figure~\ref{fig:impact_shape} (center)
which describes the same setting depicted in Figure~
\ref{fig:impact_shape} (left) but with the red line indicating the
variance of the noisy ratios.  Notice the significant differences in
error variances, with Loving county exhibiting the largest variance.
This aspect is also shown in Figure \ref{fig:p2motivation} where the
counties with ratios lying near the threshold value have higher
decisions errors than those whose ratios lies far from it.

\subsection{Predicates Composition}

The next result highlights the negative impact coming from the
composition of Boolean predicates. The following important result is
corollary of Theorem \ref{thm:and_unfairness} and provides a lower
bound on the fairness bound.
% \begin{corollary}
% \label{cor:and_unfair}
% Let mechanism $\cM$ be perfectly fair w.r.t.~to problems $P^1$ and 
% problem $P^2$. Then, $\cM$ may not be perfectly fair w.r.t. 
% problem $P = P^1 \lor P^2$.
% \end{corollary}

\begin{corollary}
\label{cor:and_unfair}
Let mechanism $\cM$ be $\alpha_k$-fair w.r.t.~to problem $P^k$ ($k\in\{1,2\}$). 
Then $\cM$ is $\alpha$-fair w.r.t.~problems $P = P^1 \lor P^2$ and $P = P^1 \land P^2$, with 
$\alpha > \max(\alpha_1, \alpha_2)$.
\end{corollary}

\begin{proof}
The proof is provided for $P = P^1 \lor P^2$. The argument for the disjunctive 
case is similar to the following one.

First the proof shows that $\alpha > \alpha_1$. By \eqref{thm5:eq} of 
Theorem \ref{thm:and_unfairness}, it follows
\begin{subequations}
  \label{eq:21}
\begin{align}
        \alpha - \alpha_1 &=  \underline{B}^1  +\alpha_2 + \underline{B}^2 - (\alpha_1 +\underline{B}^1)(\alpha_2 \!+\! \underline{B}^2) \!-\! \underline{B}^1 \underline{B}^2 \\
        & = \underline{B}^1  +\alpha_2 + \underline{B}^2 -\alpha_1 \alpha_2 - \alpha_1 \underline{B}^2  - \alpha_2 \underline{B}^1 \\
        & = \underline{B}^1( 1-\alpha_2)  + \alpha_2( 1- \alpha_1) +  \underline{B}^2(1-\alpha_1).
\end{align}
\end{subequations}
Since $\cM$ is not trivial (by assumption), we have that
$0\leq \alpha_1, \alpha_2 < 0.5$. 
Thus: 
\begin{subequations}
\label{eq:22}
\begin{align}
&\underline{B}^1( 1-\alpha_2) >0 \\
&\alpha_2( 1- \alpha_1) \geq 0\\
&\underline{B}^2(1-\alpha_1) >0.
\end{align}
\end{subequations}
Combining the inequalities in \eqref{eq:22} above with Equation 
\ref{eq:21}, results in
$$
\alpha -\alpha_1 > 0,
$$ 
which implies that $\alpha > \alpha_1$. 
An analogous argument follows for $\alpha_2$.
% Since the level of unfairness $\alpha$ in Theorem 
% \ref{thm:and_unfairness}  is a symmetric function 
% w.r.t $(\alpha_1, \underline{B^1})$  
% and $(\alpha_2, \underline{B^2})$, 
% therefore a similar proof can be performed to show 
% $\alpha > \alpha_2$. Overall, 
Therefore, $\alpha > \alpha_1$ and $\alpha >\alpha_2$, 
which asserts the claim. 
\end{proof}

\noindent
Figure \ref{fig:impact_shape} (right) illustrates Corollary
\ref{cor:and_unfair}.  It once again uses the minority language
problem $P^M$.  In the figure, each dot represents the absolute bias
$|B^i_{P^M}(\cM, \bm{x})|$ associated with a selected county. Red and
blue circles illustrate the absolute bias introduced by mechanism
$\cM$ for problem $P^1(x^{sp}) \!=\!\mathbbm{1} \{ x^{sp} \!\geq\!
10^4\}$ and $P^2(x^{sp}, x^{spe}) \!=\! \mathbbm{1}\{
\frac{x^{spe}}{x^{sp}} \!>\!0.0131\}$ respectively.  The selected
counties have all similar and small absolute bias on the two
predicates $P^1$ and $P^2$.  However, when they are combined using
logical connector $\lor$, the resulting absolute bias increases
substantially, as illustrated by the associated green circles.

{
The following analyzes an interesting difference in errors 
based on the Truth values of the composing predicates $P^1$ and $P^2$,  and shows that the highest error is achieved when they both are True for $\land$ and when they both are False for $\lor$ connectors. 
This result may have strong implications in classification tasks.

\begin{theorem}
\label{pro:nice_prop}
 Suppose mechanism $\cM$ is fair w.r.t.~predicates $P^1$ and $P^2$, 
 and consider predicate $P = P^1 \land P^2$. 
 Let $|B_P(a,b)|$ denote the absolute bias for 
 $\cM$ w.r.t.~$P$ when predicate $P^1 = a$ and predicate 
 $P^1 = b$, for $a, b \in 
 \{\textsl{True}, \textsl{False}\}$.
 Then, 
 $|B_P(\textsl{True}, \textsl{True})| \geq |B_P(a, b)|$
  for any other $a, b \in \{\textsl{True}, \textsl{False}\}$.
 \end{theorem}

\begin{proof}
Since $\cM$ is fair w.r.t.~to both predicates
$P^1$ and $P^2$, then, by (consequence of) Corollary \ref{cor:2}, 
$\forall i \in [n]$, $\cM$ absolute bias w.r.t.~ $P^1$ and $P^2$ is constant and bounded in $(0, \nicefrac{1}{2})$:
\begin{align*}
|B^i_{P^1}| = B_1 \in (0, 0.5); \qquad
|B^i_{P^2}| = B_2 \in (0, 0.5)
\end{align*} 
Given the bound above, by Lemma \ref{lem:and_formula} and for every $x \in \cX$, it is possible to derive the following sequence of relations between each combination of predicate truth values: 
% We are interested in showing that the last quantity is the largest
% among the four quantities above. From the equalities above the 
% following sequence of relations can be derived,
\begin{subequations}
\begin{align}
    \eqref{lemma1:c4} &> \eqref{lemma1:c2}; \\
    \eqref{lemma1:c4} &> \eqref{lemma1:c3}; \\
    \eqref{lemma1:c2} &> \eqref{lemma1:c1}; \\
    \eqref{lemma1:c3} &> \eqref{lemma1:c1},  \\
\end{align}
\end{subequations}
It follows immediately that case \eqref{lemma1:c4} is the largest among all the other cases, concluding the proof.
%which shows the correctness of Property \ref{pro:nice_prop} for $P = P^1 \land P^2$ . The other case  $P = P^1 \lor P^2$ can be proved similarly.
\end{proof}

\begin{theorem}
\label{pro:nice_prop2}
 Suppose mechanism $\cM$ is fair w.r.t.~predicates $P^1$ and $P^2$, 
 and consider predicate $P = P^1 \lor P^2$. 
 Let $|B_P(a,b)|$ denote the absolute bias for 
 $\cM$ w.r.t.~$P$ when predicate $P^1 = a$ and predicate 
 $P^1 = b$, for $a, b \in 
 \{\textsl{True}, \textsl{False}\}$.
 Then, 
 $|B_P(\textsl{False}, \textsl{False})| \geq |B_P(a, b)|$
  for any other $a, b \in \{\textsl{True}, \textsl{False}\}$.
 \end{theorem}
The proof follows an analogous argument to that used in the proof of Theorem \ref{pro:nice_prop}.

Figure \ref{fig:impact_group} illustrates this result on the 
Minority Language problem, 
with $P^1 = \frac{x_i^{sp}}{x_i^{s}} \!>\! 0.05$ 
$P^2 = x_i^{sp} \!>\! 10^4$, and $P^3 = \frac{x_{i}^{spe}}{x_{i}^{sp}} \!>\! 0.0131$.
It reports the decision errors on the y-axis (absolute bias). Notice that the red group is the 
most penalized in Figure \ref{fig:impact_group} (left)  and the 
least penalized in Figure \ref{fig:impact_group} (right).
\begin{figure*}[!t]
\centering
\includegraphics[height=120pt]{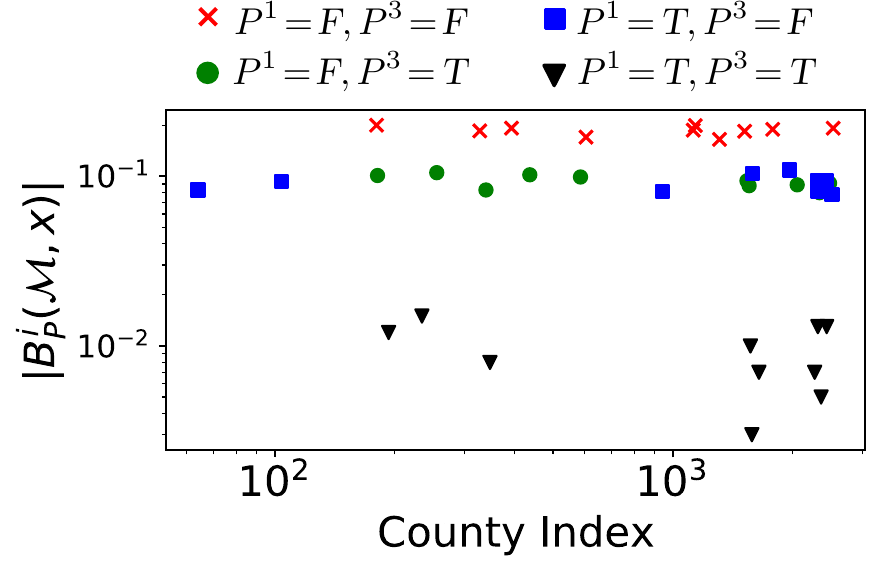}
\includegraphics[height=120pt]{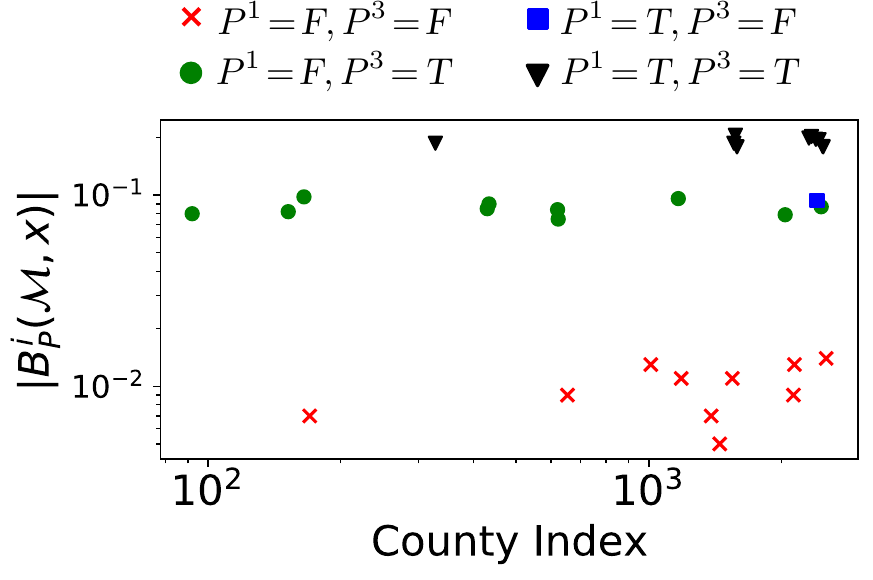}
\caption{Absolute bias (decision errors) in the Minority Language Problem: The errors are shown for four different groups of data corresponding to predicates $P = P^1 \lor P^2$ (left) 
and $P = P^1 \land P^3$ (right)}
\label{fig:impact_group}
\end{figure*}
}

\subsection{Post-Processing}

The final analysis of bias relates to the effect of post-processing
the output of the differentially private data release.  In particular,
the section focuses on ensuring non-negativity of the released
data. The discussion focuses on problem $\tfa$ but the results are,
once again, general. 

%%%%%%%%%%%%%%%%%%%%%%%%%%%%%%%%%%%%%%%%%%%%%%%%%%%%%%%%%%%%%%%%%%
% \nando{Here I am tempted to use only the input one...}

% \begin{figure}
% \centering
% %    \includegraphics[width=\linewidth]{IJCAI-p1-motivation3}
%     \includegraphics[width=\linewidth]{IJCAI_Newyork_pp_bias_pop.png}\\
%     \includegraphics[width=\linewidth]{IJCAI_Newyork_type_A_bias_pop.png}\\
%     \includegraphics[width=\linewidth]{IJCAI_Newyork_type_B_bias_pop.png}\\
%     \caption{}
%     \label{fig:p2_solutions}
% \end{figure}

The section first presents a \emph{negative result}: the application
of a post-processing operator 
$$\text{PP}^{\geq \ell}(z) \defeq \max(\ell, z)$$ 
to ensure that the result is at least $\ell$ induces a
positive bias which, in turn, can exacerbate the disparity error of
the allotment problem.

\begin{theorem}
  \label{lem:exp_clipped_lap}
  Let $\tilde{x} = x + \Lap(\lambda)$, with scale $\lambda > 0$, 
  and $\hat{x} = \text{PP}^{\geq \ell}(\tilde{x})$, with $\ell < x$, 
  be its post-processed value. Then, 
  $$ 
    \EE[\hat{x}] = x + \frac{\lambda}{2} \exp(\frac{\ell - x}{\lambda} ).
  $$
\end{theorem}

\begin{proof}
The expectation of the post-processed value $\hat{x}$ is given by: 
\begin{subequations}
\begin{align}
        E[\hat{x}] &= \int_{-\infty}^{\infty}\!\! \max(\ell, \tilde{x}) p(\tilde{x}) \;d\tilde{x}  \\
        &= \int_{-\infty}^{\ell } \!\!\max(\ell, \tilde{x}) p(\tilde{x}) \;d\tilde{x}
        + \int_{\ell}^{x } \!\!\max(\ell, \tilde{x}) p(\tilde{x}) \;d\tilde{x}+ \int_{x}^{\infty }\!\!\max(\ell, \tilde{x}) p(\tilde{x}) \;d\tilde{x}, \label{eq:clipped_lap}
\end{align}
\end{subequations}
where $p(\tilde{x}) = \nicefrac{1}{2\lambda} \exp(-\nicefrac{|\tilde{x} -x|}{\lambda} )$ is the pdf of Laplace. The following computes separately the three terms in Equation  \ref{eq:clipped_lap}:
\begin{align}
    \int_{-\infty}^{\ell}\!\! \max(\ell, \tilde{x}) p(\tilde{x}) \;d\tilde{x} 
    & = \int_{-\infty}^{\ell}\!\! \ell\, p(\tilde{x}) \;d\tilde{x} = \ell \int_{-\infty}^{\ell}\!\!  p(\tilde{x}) \;d\tilde{x} 
     = \frac{1}{2}  \ell \exp( \frac{\ell -x}{\lambda} ) 
     \label{clip_eq:1}\\
  \int_{\ell}^{x }\!\! \max(\ell, \tilde{x}) p(\tilde{x}) \;d\tilde{x} 
  &= \int_{\ell}^{x }\!\! \tilde{x} p(\tilde{x}) \;d\tilde{x}  
=\frac{1}{2}(x- \lambda) -\frac{1}{2} (\ell-\lambda) \exp( \frac{\ell-x}{\lambda}) \label{clip_eq:2}\\
   \int_{x}^{\infty }\!\! \max(\ell, \tilde{x})  p(\tilde{x}) \;d\tilde{x}    
   &= \int_{x}^{\infty }\!\! \tilde{x} p(\tilde{x}) \;d\tilde{x}
    =\frac{1}{2} (x + \lambda).
\label{clip_eq:3}
\end{align}

Combining equations \eqref{clip_eq:1}--\eqref{clip_eq:3} with \eqref{eq:clipped_lap}, gives
$$ E[\hat{x}] = x + \frac{\lambda}{2} \exp(\frac{\ell-x}{\lambda} ).$$
\end{proof}

Lemma \ref{lem:exp_clipped_lap} indicates the presence of positive
bias of post-processed Laplace random variable when ensuring
non-negativity, and that such bias is $B^i(\cM, \bm{x}) =
\EE[\hat{x_i}] - x_i =  %= \frac{\lambda}{2}
\exp(\frac{\ell-x_i}{\lambda}) \leq \nicefrac{\lambda}{2} $ for $\ell
\leq x_i$.  As shown in Figure \ref{fig:p1motivation} the effect of
this bias has a negative impact on the final disparity of the
allotment problem, where smaller entities have the largest bias (in
the Figure $\ell=0$).

{

The remainder of the section discusses positive results for two additional classes of post-processing: (1) The integrality constraint program $\text{PP}^{\mathbb{N}}(z)$, which enforces the integrality of the released values, and (2) The sum-constrained  constrained program $\text{PP}^{\sum_S}(z)$, which enforces a linear constraint on the data. The following results show that these
post-processing steps do not contribute to further biasing the decisions.

\paragraph{Integrality Post-processing }

The integrality post-processing $\text{PP}^{\mathbb{N}}(z)$ is used when the released data are integral quantities. The following post-processing step, based on stochastic rounding produces integral quantities:
\begin{equation}
    \text{PP}^{\mathbb{N}}(z) = \begin{cases} \floor{z}  \ \mbox{w.p.:} \  1 -(z- \floor{z}) \\
                               \floor{z} +1 \ \ \mbox{w.p.:} \  z - \floor{z}
    \end{cases}
\end{equation}
It is straightforward to see that the above is an unbias estimator:
$\mathbb{E} [\text{PP}^{\mathbb{N}}(\tilde{x})] = \tilde{x} $ and thus, no it introduces no additional bias to $\text{PP}^{\mathbb{N}}(\tilde{x})$.

\paragraph{Sum-constrained Post-processing}
The sum-constrained post-processing  $\text{PP}^{\sum_S}(z)$ is expressed through the following constrained optimization problem:
\begin{equation}
\label{eq:sum_opt}
\min_{\hat{z} } \left\| \hat{z} - z \right\|^2_2\\ \ 
\mbox{s.t}: \ \bm{1}^T z = S
\end{equation}
This class of constraints is typically enforced when the private outcomes are required to match some fixed resource to distribute. For example, the outputs of the allotment problem $P^F$ should be such that the total budget is allotted, and thus $\sum_i P^F_i(\tilde{x}) = 1$. 

\begin{theorem}
\label{thm:sum_postopt}
Consider an $\alpha$-fair mechanism $\cM$ w.r.t.~problem $P$. Then $\cM$ is also $\alpha$-fair w.r.t.~problem $\text{PP}^{\sum_S}(P)$.
\end{theorem}
The following relies on a result by Zhu et al.~\cite{zhu:20_postdp}. 

\begin{proof}
Denote $z_i$ and $\hat{z}_i$ as for $= P_i(\bm{\tilde{x}})$ and 
$\text{PP}^{\sum_S}(z_i)$, respectively.
Note that problem \eqref{eq:sum_opt} is convex and its unique minimizer is $\hat{z}_i = z_i + \eta$ with $\eta = \frac{S - \sum_{i} z_i}{n}$. Its expected value is:
\begin{subequations}
\begin{align}
  \mathbb{E} [\hat{z}_i] 
  &= \mathbb{E}[z_i + \eta] = \mathbb{E}\left[z_i + \frac{S -\sum_{j\neq i} z_j}{n} \right]\\
  &= \frac{n-1}{n} \mathbb{E}\left[z_i\right] - \frac{1}{n} 
    \sum_{j \neq i} \mathbb{E}\left[z_{j}\right]+\frac{S}{n}\\
  &= \frac{n-1}{n} \left(z_i + B^i_P \right) - 
    \frac{1}{n}  \left( \sum_{j \neq i} z_j  + B^j_P \right)+ 
    \frac{S}{n}\\
  &= \frac{n-1}{n} \left( z_i + B^i_P \right) - \frac{1}{n}  
  \left(S - z_i  +\sum_{j \neq i}  B^j_P \right) + \frac{S}{n}\\
  &=  z_i + \frac{\sum_{j \neq i} \left(B^i_P - B^j_P \right)}{n}.
\end{align}
\end{subequations}
The above follows from linearity of expectation and the last equality indicates that the bias of entity $i$ under sum-constrained post-processing is $B^i_{ \text{PP}^{\sum_S}(P)} = \frac{\sum_{j \neq i} (B^i_P - B^j_P)}{n}  $. 
Thus, the fairness bound  $\alpha'$ attained after post-processing is:
\begin{subequations}
\begin{align}
    \alpha' 
    &= \max_{i,k} \left| B^i_{ \text{PP}^{\sum_S}(P)} - B^k_{ \text{PP}^{\sum_S}(P)} \right| \\
    &= \max_{i,k} \frac{\sum_{j \neq i} (B^i_P - B^j_P)}{n} - \frac{\sum_{j \neq k} (B^k_P - B^j_P)}{n} \\
    &= \max_{i,k}  |B^i_P - B^k_P| = \alpha
\end{align}
\end{subequations}
Therefore, the sum-constrained post-processing does not introduce additional unfairness to mechanism $\cM$.
\end{proof}
}

\iffalse%%%%%%%%%%%%%%%%%%%%%%%%%%%%%%%%%%%%%%%%%%%%%%%%%%%%%%%%%%%%%%
\begin{example}
\label{ex4}
The following \textbf{synthetic} example provides a further illustration
about the negative impact of $\text{PP}^{\geq 1}$ to fairness. 
The dataset uses $n = 200$ and sets $x_i = i \forall i \in[200]$. 
The post-processing step $\text{PP}^{\geq 0}(\tilde{x}_i)$ is applied 
to the outputs $x_i$ of the Laplace mechanism $\cM$ applied on data $x_i$.
\begin{figure}[h!]
  \includegraphics[width=0.5\textwidth]{impact_pos.pdf}
    \caption{Impact of lower bound truncation post-processing step $\tilde{x}  = \max(\tilde{x}, L)$  for $\epsilon = 0.1, 0.5 ,1$}
    \label{fig:impact_pos}
\end{figure}
Figure \ref{fig:impact_pos} illustrates how the disparity error varies
across the entities, with smaller entities suffers from higher errors. 
Intuitive smaller entities have higher probability to generate private 
values that are less than the lower bound 1, and thus will be affected 
more by the post-processing step. It is easy to observe that, due to 
post-processing, $\mathbb{E}[\tilde{x}] >> x $ especially for high 
privacy regimes (small $\epsilon$ values).
\end{example}
\fi%%%%%%%%%%%%%%%%%%%%%%%%%%%%%%%%%%%%%%%%%%%%%%%%%%%%%%%%%%%%%%%%%%

\paragraph{Discussion} The results highlighted in this section are both surprising 
and significant. They show that {\em the motivating allotment problems
  and decision rules induce inherent unfairness when given as input
  differentially private data}. This is remarkable since the resulting
decisions have significant societal, economic, and political impact on
the involved individuals: federal funds, vaccines, and therapeutics 
may be unfairly allocated, minority language voters may be
disenfranchised, and congressional apportionment may not be fairly
reflected. The next section identifies a set of guidelines to mitigate
these negative effects.

%%%%%%%%%%%%%%%%%%%%%%%%%%%%%%%%%%%%%%%%%%%%%%%%%%%%%%%%%%%%%%%%%%%%%
\section{Mitigating Solutions}
%%%%%%%%%%%%%%%%%%%%%%%%%%%%%%%%%%%%%%%%%%%%%%%%%%%%%%%%%%%%%%%%%%%%%
\subsection{The Output Perturbation Approach} 

This section proposes three guidelines that may be adopted to mitigate 
the unfairness effects presented in the paper, with focus on the motivating 
allotments problems and decision rules.

A simple approach to mitigate the fairness issues discussed is to 
recur to \emph{output perturbation} to randomize the outputs of problem 
$P_i$, rather than its inputs, using an unbiased mechanism. 
Injecting noise directly after the computation of the outputs 
$P_i(\bm{x})$, ensures that the result will be unbiased. However, 
this approach has two shortcomings. First, 
it is not applicable to the context studied in this paper, where a data 
agency desires to release a privacy-preserving data set $\tilde{\bm{x}}$ 
that may be used for various decision problems. 
Second, computing the sensitivity of the problem $P_i$ may be hard, it 
may require to use a conservative estimate, or may even be impossible, 
if the problem has unbounded range. 
A conservative sensitivity implies the introduction of significant 
loss in accuracy, which may render the decisions unusable in practice.

\subsection{Linearization by Redundant Releases}
A different approach considers modifying on the decision problem $P_i$
itself.  Many decision rules and allotment problems are designed in an
ad-hoc manner to satisfy some property on the original data, e.g.,
about the percentage of population required to have a certain
level of education.  Motivated by Corollaries \ref{cor:1} and
\ref{cor:2}, this section proposes guidelines to modify the original
problem $P_i$ with the goal of reducing the unfairness effects introduced
by differential privacy. 

The idea is to use a linearized version $\bar{P}_i$ of problem $P_i$. 
While many linearizion techniques exists \cite{rebennack2020piecewise}, and are often 
problem specific, the section focuses on a linear proxy 
$\bar{P}_i^F$ to problem $\tfa_i$ that can be obtained by enforcing a 
redundant data release. While the discussion focuses on problem $\tfa_i$, 
the guideline is general and applies to any allotment problem with similar 
structure. 

Let $Z = \sum_i a_i x_i$. Problem $\tfa_i(\bm{x}) \!=\! \nicefrac{a_i x_i}{Z}$ 
is linear w.r.t.~the inputs $x_i$ but non-linear w.r.t.~$Z$. 
However, releasing $Z$, in addition to releasing the privacy-preserving 
values $\tilde{\bm{x}}$, would render $Z$ a constant rather than a problem 
input to $\tfa$.
To do so, $Z$ can either be released publicly, at cost of a (typically small) 
privacy leakage or by perturbing it with fixed noise. The resulting 
linear proxy allocation problem $\bar{P}^F_i$ is thus linear in the inputs $\bm{x}$. 

\begin{figure}
\centering
    \includegraphics[width=0.85\linewidth]{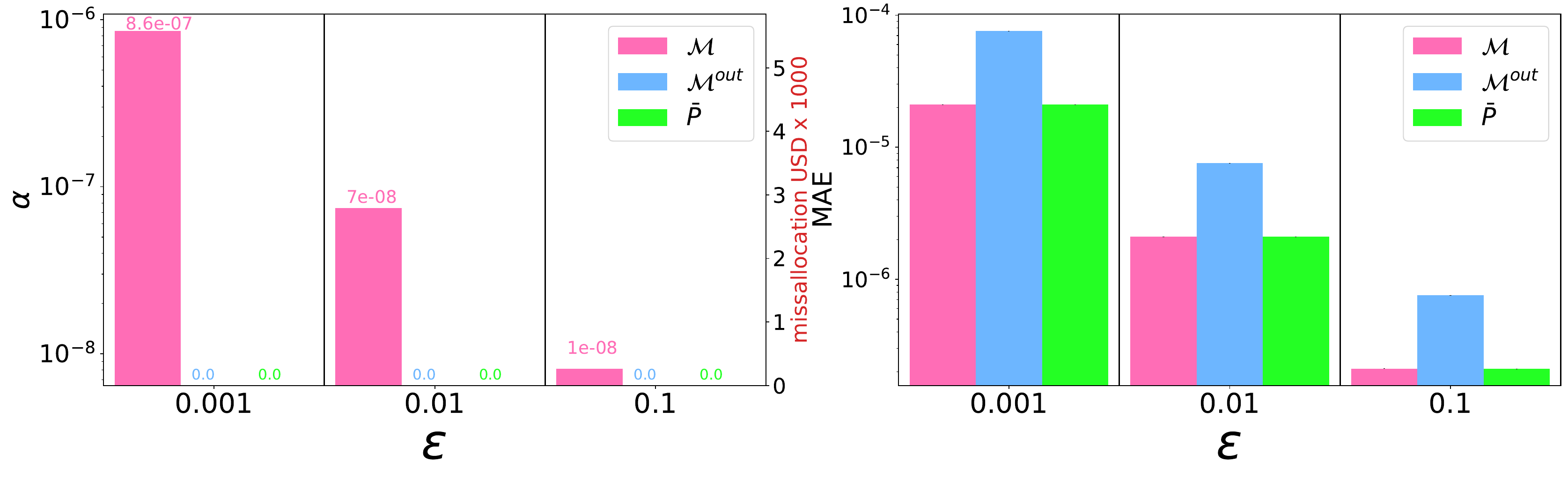}
    \caption{Linearization by redundant release: Fairness and error comparison.
    }
    \label{fig:p2_solutions}
\end{figure}
Figure \ref{fig:p2_solutions} illustrates this approach in practice.
The left plot shows the fairness bound $\alpha$ and the right plot
shows the empirical mean absolute error $\frac{1}{m}\sum_{k=1}^{m}
|P_i(\bm{x}^k) - P_i(\tilde{\bm{x}}^m)|$, obtained using $m=10^4$
repetitions, when the DP data $\tilde{\bm{x}}$ is applied to (1) the
original problem $P$, (2) its linear proxy $\bar{P}$, and (3) when
output perturbation (denoted $\cM^{\text{out}}$) is adopted.  The
number on top of each bar reports the fairness bounds, and emphasize
that the proposed remedy solutions achieve perfect-fairness.  Notice
that the proposed linear proxy solution can reduce the fairness
violation dramatically while retaining similar errors.
% The figure also compare against an \emph{output perturbation} method, 
% where noise is injected directly after the computation of the problem 
% $P$ rather than to the input data $\bm{x}$. While this is not feasible
% in the application contest studied in the paper, output perturbation 
% guarantees that the result will be unbiased. 
While  the output perturbation method reduces the disparity error, it also incurs
significant errors that make the approach rarely usable in practice.
%let alone the fact that the range of $P$ is required to be small 
%enough to derive a small sensitivity $\Delta$. 
% {The next section also discuss a solution based on a 
% piecewise linear proxy function for the more complex decision rule $P^M$.}

{
\paragraph{Learning Piece-wise Linear proxy-functions}

Due to the discontinuities arising in decision rules (see for example
problem $\eqref{p:coverage}$), it is substantially more challenging to develop mitigation strategies than in allotment problems. In
particular, the discontinuities present in problem $P^M$ render the
use of linear proxies ineffective.

The following strategy combines two ideas: (1) partitioning, in a privacy-preserving fashion, the original problem
into subproblems that are locally continuous and amenable to
linearizations with low accuracy loss, and (2) the systematic learning of linear proxies. 
More precisely, the idea is to partition the input values $\bm{x}$ into several groups $\bm{x}_1, \ldots \bm{x}_G$ (e.g.,
individuals from the same state or from states of similar magnitude)
and to approximate subproblem $P^M_i(\bm{x}_k)$ with a linear proxy
$\bar{P}^M_i(\bm{x}_k)$ for each group $k \in [G]$.  The resulting
problem $\bar{P}^M_i$ then becomes a piecewise linear function that
approximates the original problem $P^M_i$. 

Rather than using an ad-hoc method to linearize problem $P^M$, the paper proposes to obtain it by fitting a linear model to the 
 data  $\bm{x}_k$ of each group $k \in [G]$. 
Figure \ref{fig:p1_piecewise_private} presents results for problem $P^M$. Each subgroup is trained using features $\{x^{spe},
x^{sp}, x^s\}$ and the resulting model coefficients are used to
construct the proxy linear function for the subproblems
$\bar{P}^M_i(\bm{x}_G)$.  The results use the value
$x^{sp}$ to partition the dataset into $9$ groups of approximately
equal size. To ensure privacy, the grouping is executed using
privacy-preserving $x^{sp}$ values. 
Figure \ref{fig:p1_piecewise_private} compares the original problem $P$, a proxy-model $\bar{P}_{LR}$ whose pieces are learned using linear
regression (LR), and a proxy model $\bar{P}_{SVM}$ whose pieces are
learned using a linear SVM model. All three problems take as input the
private data $\tilde{\bm{x}}$ and are compared with the original
version of the problem $P$.  The x-axis shows the range of $x^{sp}$
that defines the partition, while the y-axis shows the fairness bound
$\alpha$ computed within each group. \emph{The positive effects of the
  proposed piecewise linear proxy problem are dramatic}. The fairness
violations decrease significantly when compared to those obtained by
the original model. the fairness violation of the SVM model is
typically lower than that obtained by the LR model, and this may be
due to the accuracy of the resulting model -- with SVM reaching higher
accuracy than LR in our experiments. Finally, as the population size
increases, the fairness bound $\alpha$ decreases and emphasizes
further the largest negative impact of the noise on the smaller
counties.

\begin{figure}[tb!]
\centering
  \includegraphics[width=\textwidth]{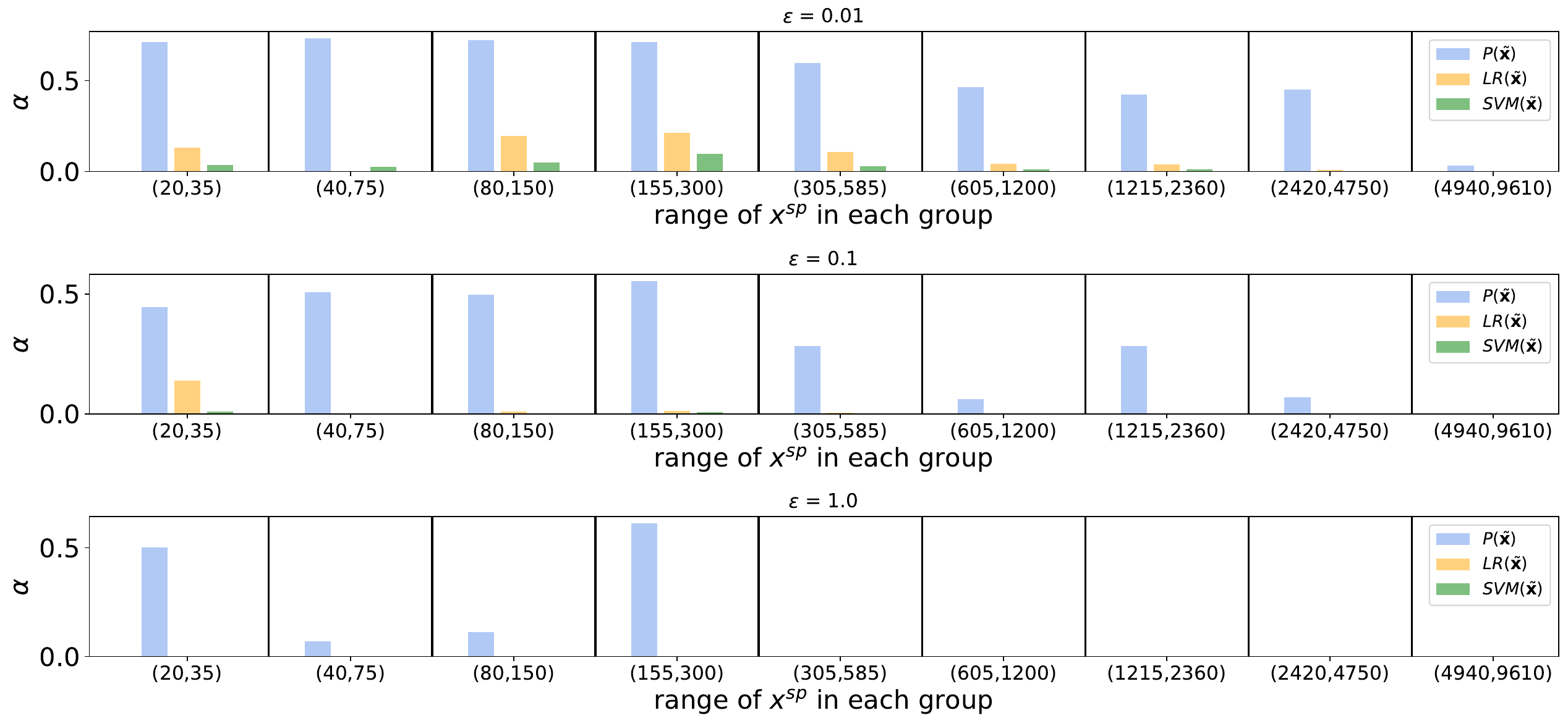}
    \caption{Linearization by redundant release: Fairness comparison.}
    \label{fig:p1_piecewise_private}
\end{figure}
}

\smallskip
It is important to note that the experiments above use a data release 
mechanism $\cM$ that applies no post-processing. A discussion about the 
mitigating solutions for the bias effects caused by post-processing 
is presented next. 

\subsection{Modified Post-Processing}

This section introduces a simple, yet effective, solution to mitigate 
the negative fairness impact of the non-negative post-processing. 
The proposed solution operates in 3 steps:
It first (1) performs a non-negative post-processing of the privacy-preserving 
input $\tilde{x}$ to obtain value $\bar{x} \!=\! \text{PP}^{\geq \ell}(\tilde{x})$.
Next, (2) it computes $\bar{x}_T \!=\! \bar{x} \!-\! \frac{T}{\bar{x} +1 - \ell}$.
Its goal is to correct the error introduced by the post-processing operator, 
which is especially large for quantities near the boundary $\ell$. Here $T$ is a 
\emph{temperature} parameter that controls the strengths of the correction.
This step reduces the value $\bar{x}$ by quantity $\frac{T}{ \bar{x} + 1 -\ell}$. 
The effect of this operation is to reduce the expected value 
$\mathbb{E}[\bar{x}]$ by larger (smaller) amounts as $x$ get closer (farther) 
to the boundary value $\ell$.  Finally, (3) it ensures that the final estimate 
is indeed lower bounded by $\ell$, by computing $\hat{x} = \max(\bar{x}_T, \ell)$.

The benefits of this approach are illustrated in Figure \ref{fig:remedy_pos},
which show the absolute bias $|B^i_{\tfa}|$ for the Title 1 fund allocation 
problem that is induced by the original mechanism $\cM$ with standard post-processing 
$\text{PP}^{\geq 0}$ and by the proposed modified post-processing for different
temperature values $T$.
\begin{figure}[tb!]
\centering  \includegraphics[width=\linewidth]{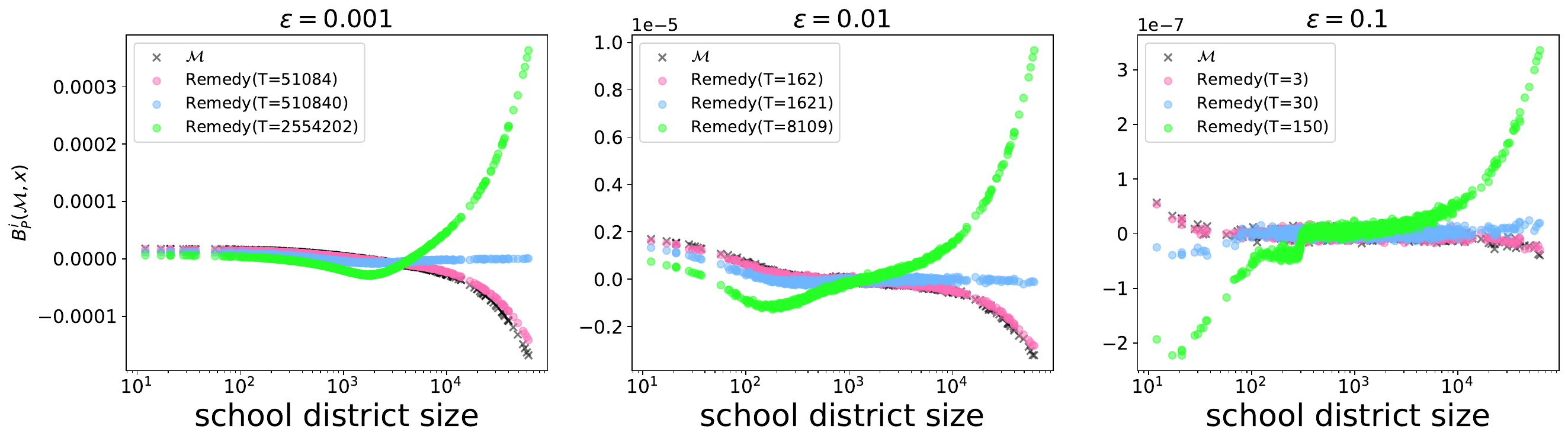}
    \caption{Modified post-processing: Unfairness reduction.}
    \label{fig:remedy_pos}
\end{figure}
The figure illustrates the role of the temperature $T$ in the disparity errors.
Small values $T$ may have small impacts in reducing the disparity errors, while
large $T$ values can introduce errors, thus may exacerbate unfairness.
The optimal choice for $T$ can be found by solving the following:
\begin{align}
    T^* = \argmin_{T} \big( \max_{\bm{x} \geq \ell} 
    |\mathbb{E}[\hat{\bm{x}}_T] - \bm{x}| -\min_{\bm{x} \geq \ell}|\mathbb{E}[\hat{\bm{x}}_T] - \bm{x}| \big),
\end{align}
where $\hat{\bm{x}}_T$ is a random variable obtained by the proposed 3 step 
solution, with temperature $T$. The expected value of $\hat{\bm{x}}$ can be approximated via sampling.
% \rev{A discussion on the sampling errors introduced by this approach 
% is reported in the appendix}.
Note that naively finding the optimal $T$ may require access to 
the true data. Solving the problem above in a privacy-preserving way is beyond
the scope of the paper and the subject of future work.
% Note that the inputs to the problem above are differentially private, 
% and thus this step does not further leak privacy.

The reductions in the fairness bound $\alpha$ for problem $\tfa$ are
reported in Figure \ref{fig:remedy_pos_2} (left), while Figure
\ref{fig:remedy_pos_2} (right) shows that this method has no
perceptible impact on the mean absolute error. Once again, these
errors are computed via sampling and use $10^4$ samples.
\begin{figure}[tb!]
  \centering 
  \includegraphics[width=0.8\linewidth]{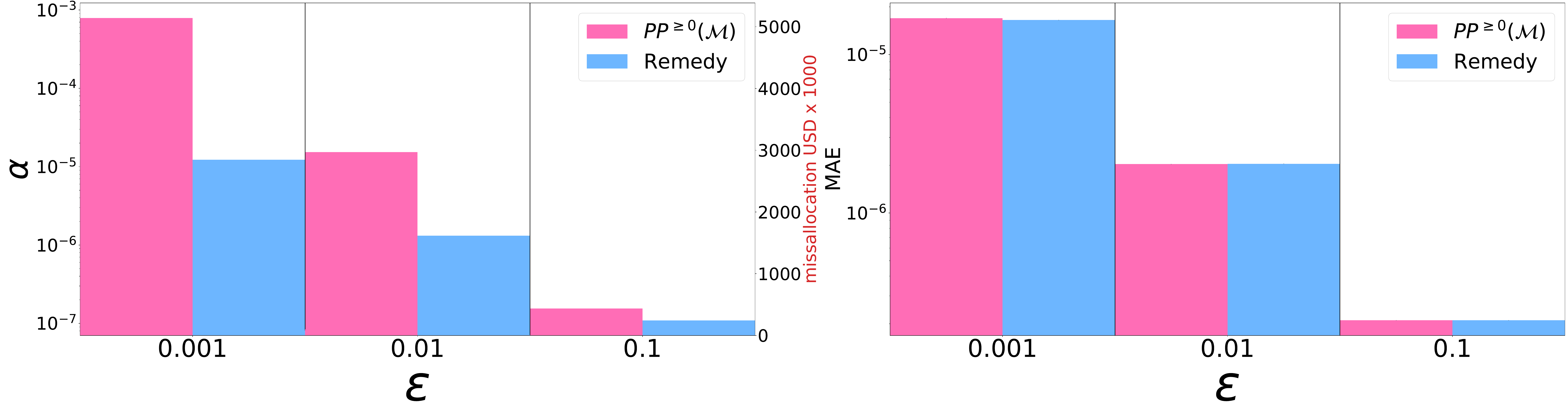}
    \caption{Modified post-processing on problem $\tfa$.}
    \label{fig:remedy_pos_2}
\end{figure}

\subsection{Fairness Payment}

Finally, this section focuses on allotment problems, like $\tfa$, that
distribute a budget $B$ among $n$ entities, and where the allotment
for entity $i$ represents the fraction of budget $B$ it
expects. Differential privacy typically implements a postprocessing
step to renormalize the fractions so that they sum to 1. This
normalization, together with nonnegativity constraints, introduces a
bias and hence more unfairness. One way to alleviate this problem is
to increase the total budget $B$, and avoiding the normalization.
This section quantifies the cost of doing so: it defines the \emph{cost of
  privacy}, which is the increase in budget $B^+$ required to achieve
this goal.

\begin{definition}[Cost of Privacy]
Given problem $P$, that distributes budget $B$ among $n$ entities, 
data release mechanism $\cM$, and dataset $\bm{x}$, the cost of privacy is: 
%We define the cost of privacy  for non-negative bias $b_P$  of problem $P$ under mechanism $\cM$ to be the total additional amount of resources should be further allocated, so that each entity received a non-negative bias w.r.t private allocation. In particular,
%$$b_P = \sum_{i \in [n]}   b_i  + \sum_{ b_i <0} |b_i| ,$$
\[
  B^+ = 
  % (\sum_{i \in [n]} B^i_P(\cM, \bm{x}) +
   \textstyle\sum_{i \in I^-} |B^i_P(\cM, \bm{x})| \times B
\]
with $I^- = \{i \, : \, B^i_P(\cM, \bm{x}) < 0\}$.
\end{definition}
%That is, $B^+$ is the total negative bias 

\noindent Figure \ref{fig:privacy_cost} illustrates the cost of privacy, 
in USD, required to render each county in the state of New York not 
negatively penalized by the effects of differential privacy. 
The figure shows, in decreasing order, the different costs associated 
with a mechanism $P^F(\text{PP}^{\geq 0} (\bm{x}))$ that applies a post-processing step, 
one $P^F(\bm{x})$ that does not apply post-processing, and one that uses
a linear proxy problem $\bar{P}^F(\bm{x})$. 
% As explained above for most input-perturbation mechanism in which for 
% any random seed of the mechanism, a fixed (no more no less) budget $B$ 
% is distributed then total bias $\sum_i b_i =0$. 
% Thus the cost of privacy for non-negative bias is simply 
% $b_P =\sum_{ b_i <0} |b_i| $. 

% While $B^+$ may be difficult to compute for arbitrary mechanism $\cM$ and
% decision problems $P$, when the linearization remedy solution is 
% adopted, the resulting problem $\bar{P}$ 
% it is easy. The reason is that  since $f(Z) = \frac{1}{Z}$ is a 
% concave function, so by Jensen inequality 
% $\mathbb{E}[\frac{1}{\tilde{Z}} ]  \geq  \frac{1}{\mathbb{E}[\tilde{Z}]} 
% = \frac{1}{Z}$, so the bias $b_i = a_i x_i \big( \mathbb{E}[\frac{1}{\tilde{Z}} 
% - \frac{1}{Z}] \big) >0 $ under $\bar{P}$ is always positive for any $i$.  

% The cost of privacy for non-negative bias $b_P$ is also the cost 
% of privacy $c_P = \frac{B}{\epsilon^2_1} \frac{1}{Z^3}$
\begin{figure}[tb!]
	\centering
    \includegraphics[width=0.6\linewidth]{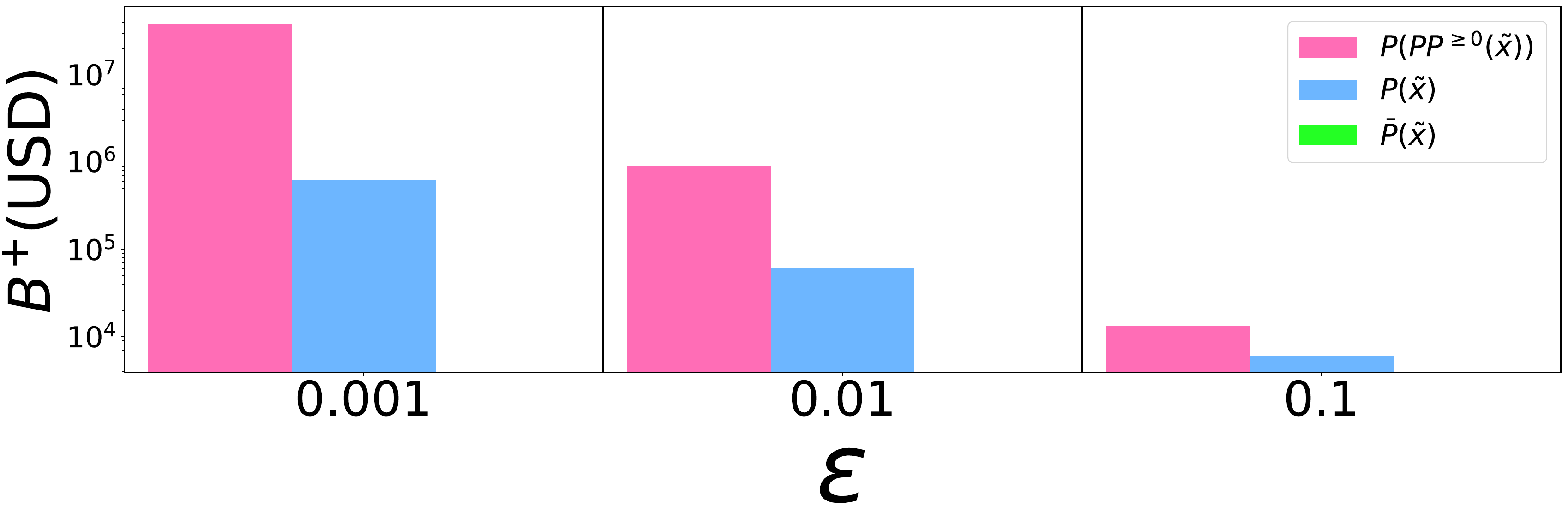}
    \caption{Cost of privacy on problem $\tfa$.}
    \label{fig:privacy_cost}
\end{figure}

\section{Related Work}
The literature on DP and algorithmic fairness 
is extensive and the reader is referred to, respectively, 
\cite{Dwork:13,vadhan2017complexity,dwork2017exposed} and \cite{barocas2017fairness,mehrabi2019survey} 
for surveys on these topics.
However, privacy and fairness have been studied mostly in isolation
with a few exceptions. Cummings et al.~\cite{cummings:19} consider the
tradeoffs arising between differential privacy and equal opportunity,
a fairness concept that requires a classifier to produce equal true
positive rates across different groups. They show that there exists no
classifier that simultaneously achieves $(\epsilon,0)$-differential
privacy, satisfies equal opportunity, and has accuracy better than a
constant classifier.  Ekstrand et al.~\cite{ekstrand:18} raise
questions about the tradeoffs involved between privacy and fairness,
and Jagielski et al.~\cite{jagielski:18} shows two algorithms that
satisfy $(\epsilon,\delta)$-differential privacy and equalized odds.
Although it may sound like these algorithms contradict the
impossibility result from \cite{cummings:19}, it is important to note
that they are not considering an $(\epsilon, 0)$-differential privacy
setting. Tran et al.~\cite{Tran:AAAI21} developed a differentially private
learning approach to enforce several group fairness notions using a
Lagrangian dual method. 
Zhu et al.~\cite{Zhu:AAAI21} studied the bias and variance induced by several important classes of post-processing and that the resulting bias can also have some disproportionate impact on the outputs. 
Pujol et al.~\cite{pujol:20} were seemingly
first to show, empirically, that there might be privacy-fairness
tradeoffs involved in resource allocation settings. In particular,
for census data, they show that the noise added to achieve
differential privacy could disproportionately affect some groups over
others. 
This paper builds on these empirical observations and provides a step
towards a deeper understanding of the fairness issues arising when
differentially private data is used as input to decision problems. 
This work is an extended version of \cite{Tran:IJCAI21}.

%%%%%%%%%%%%%%%%%%%%%%%%%%%%%%%%%%%%%%%%%%%%%%%%%%%%%%%%%%%%%%%%%%%%%
\section{Conclusions}
%%%%%%%%%%%%%%%%%%%%%%%%%%%%%%%%%%%%%%%%%%%%%%%%%%%%%%%%%%%%%%%%%%%%%
This paper analyzed the disparity arising in decisions granting 
benefits or privileges to groups of people when these decisions 
are made adopting differentially private statistics about these groups.
% The study was motivated by the use of census data for critical 
% resource allocations problems, and studied two problem settings: 
% allotment problems, such as those that allot funds to states and districts, 
% and decision rules, such as those that grants minority language voting rights. 
It first characterized the conditions for which allotment 
problems achieve finite fairness violations and bound the fairness 
violations induced by important components of decision rules, including
reasoning about the composition of Boolean predicates under logical operators. 
Then, the paper analyzed the reasons for disparity errors arising in 
the motivating problems and recognized the problem structure, the predicate
composition, and the mechanism post-processing, as paramount to the bias 
and unfairness contribution. 
Finally, it suggested guidelines to act on the decision problems
or on the mechanism (i.e., via modified post-processing steps) to mitigate
the unfairness issues. 
The analysis provided in this paper may provide useful guidelines 
for policy-makers and data agencies for testing the fairness and bias 
impacts of privacy-preserving decision making.

% The analysis provided in this paper  may have a significant effect 
% on several critical domains with broad impact on society and economy.

% The results highlighted in the paper shows that the motivating allotment 
% problems and decision rules induce inherent unfairness when given as input
% differentially private data. This result is both surprising and significant 
% since the resulting decisions can have significant societal, economic, 
% and political impact on the involved individuals.%: federal funds may be misallocated causing 
%unfair representation, minority language voters may be disenfranchised, 
%and congressional apportionment may not be fairly reflected.}

\bibliographystyle{plain}
\bibliography{lib}

\begin{thebibliography}{10}

\bibitem{title13}
Title 13.
\newblock Title~13, u.s.~code.
\newblock
  \url{www.census.gov/history/www/reference/privacy_confidentiality/title_13_us_code.html},
  2006.
\newblock Accessed: 2021-01-15.

\bibitem{abowd2018us}
John~M Abowd.
\newblock The us census bureau adopts differential privacy.
\newblock In {\em Proceedings of the 24th ACM SIGKDD International Conference
  on Knowledge Discovery \& Data Mining}, pages 2867--2867, 2018.

\bibitem{abowd2019economic}
John~M Abowd and Ian~M Schmutte.
\newblock An economic analysis of privacy protection and statistical accuracy
  as social choices.
\newblock {\em American Economic Review}, 2019.

\bibitem{barocas2017fairness}
Solon Barocas, Moritz Hardt, and Arvind Narayanan.
\newblock Fairness in machine learning.
\newblock {\em Advances in neural information processing systems ({NeurIPS})
  tutorial}, 1:2, 2017.

\bibitem{cummings:19}
Rachel Cummings, Varun Gupta, Dhamma Kimpara, and Jamie Morgenstern.
\newblock On the compatibility of privacy and fairness.
\newblock In {\em Adjunct Publication of the 27th Conference on User Modeling,
  Adaptation and Personalization}, pages 309--315, 2019.

\bibitem{dwork:06}
Cynthia Dwork, Frank McSherry, Kobbi Nissim, and Adam Smith.
\newblock Calibrating noise to sensitivity in private data analysis.
\newblock In {\em Theory of cryptography conference}, pages 265--284. Springer,
  2006.

\bibitem{Dwork:13}
Cynthia Dwork and Aaron Roth.
\newblock The algorithmic foundations of differential privacy.
\newblock {\em Theoretical Computer Science}, 9(3-4):211--407, 2013.

\bibitem{dwork2017exposed}
Cynthia Dwork, Adam Smith, Thomas Steinke, and Jonathan Ullman.
\newblock Exposed! a survey of attacks on private data.
\newblock {\em Annual Review of Statistics and Its Application}, 4:61--84,
  2017.

\bibitem{ekstrand:18}
Michael~D Ekstrand, Rezvan Joshaghani, and Hoda Mehrpouyan.
\newblock Privacy for all: Ensuring fair and equitable privacy protections.
\newblock In {\em Conference on Fairness, Accountability and Transparency},
  pages 35--47, 2018.

\bibitem{erlingsson2014rappor}
{\'U}lfar Erlingsson, Vasyl Pihur, and Aleksandra Korolova.
\newblock Rappor: Randomized aggregatable privacy-preserving ordinal response.
\newblock In {\em Proceedings of the 2014 ACM SIGSAC conference on computer and
  communications security}, pages 1054--1067. ACM, 2014.

\bibitem{Fioretto:AIJ21}
Ferdinando Fioretto, Pascal {Van Hentenryck}, and Keyu Zhu.
\newblock Differential privacy of hierarchical census data: An optimization
  approach.
\newblock {\em Artificial Intelligence}, pages 639--655, 2021.

\bibitem{GDPR}
GDPR.
\newblock What is gdpr, the eu’s new data protection law?
\newblock \url{https://gdpr.eu/what-is-gdpr}, 2020.
\newblock Accessed: 2021-01-15.

\bibitem{jagielski:18}
Matthew Jagielski, Michael Kearns, Jieming Mao, Alina Oprea, Aaron Roth, Saeed
  Sharifi-Malvajerdi, and Jonathan Ullman.
\newblock Differentially private fair learning.
\newblock {\em arXiv preprint arXiv:1812.02696}, 2018.

\bibitem{uber}
Noah Johnson, Joseph~P Near, and Dawn Song.
\newblock Towards practical differential privacy for sql queries.
\newblock {\em Proceedings of the VLDB Endowment}, 11(5):526--539, 2018.

\bibitem{kairouz2015composition}
Peter Kairouz, Sewoong Oh, and Pramod Viswanath.
\newblock The composition theorem for differential privacy.
\newblock In {\em International conference on machine learning}, pages
  1376--1385. PMLR, 2015.

\bibitem{pujol:20}
Satya Kuppam, Ryan Mckenna, David Pujol, Michael Hay, Ashwin Machanavajjhala,
  and Gerome Miklau.
\newblock Fair decision making using privacy-protected data, 2020.

\bibitem{mehrabi2019survey}
Ninareh Mehrabi, Fred Morstatter, Nripsuta Saxena, Kristina Lerman, and Aram
  Galstyan.
\newblock A survey on bias and fairness in machine learning.
\newblock {\em arXiv preprint arXiv:1908.09635}, 2019.

\bibitem{rebennack2020piecewise}
Steffen Rebennack and Vitaliy Krasko.
\newblock Piecewise linear function fitting via mixed-integer linear
  programming.
\newblock {\em INFORMS Journal on Computing}, 32(2):507--530, 2020.

\bibitem{rogers2020linkedin}
Ryan Rogers, Subbu Subramaniam, Sean Peng, David Durfee, Seunghyun Lee,
  Santosh~Kumar Kancha, Shraddha Sahay, and Parvez Ahammad.
\newblock Linkedin's audience engagements api: A privacy preserving data
  analytics system at scale.
\newblock {\em arXiv preprint arXiv:2002.05839}, 2020.

\bibitem{covid}
Lisa Simunaci.
\newblock Pro rata vaccine distribution is fair, equitable.
\newblock \url{t.ly/sDa9}, 2021.

\bibitem{Sonnenberg:16}
W.~Sonnenberg.
\newblock Allocating grants for title i.
\newblock {\em U.S.~Department of Education, Institute for Education Science},
  2016.

\bibitem{apple}
Apple Differential~Privacy Team.
\newblock Learning with privacy at scale.
\newblock {\em Apple Machine Learning Journal}, 1(8), 2017.

\bibitem{Tran:AAAI21}
Cuong Tran, Ferdinando Fioretto, and Pascal {Van Hentenryck}.
\newblock Differentially private and fair deep learning: A lagrangian dual
  approach.
\newblock In {\em Proceedings of the AAAI Conference on Artificial Intelligence
  {(AAAI)}}, page (to appear), 2021.

\bibitem{Tran:IJCAI21}
Cuong Tran, Ferdinando Fioretto, Pascal {Van Hentenryck}, and Zhiyan Yao.
\newblock Decision making with differential privacy under the fairness lens.
\newblock In {\em Proceedings of the International Joint Conference on
  Artificial Intelligence {(IJCAI)}}, page (to appear), 2021.

\bibitem{vadhan2017complexity}
Salil Vadhan.
\newblock The complexity of differential privacy.
\newblock In {\em Tutorials on the Foundations of Cryptography}, pages
  347--450. Springer, 2017.

\bibitem{zhu:20_postdp}
Keyu Zhu, Pascal {Van Hentenryck}, and Ferdinando Fioretto.
\newblock Bias and variance of post-processing in differential privacy, 2020.

\bibitem{Zhu:AAAI21}
Keyu Zhu, Pascal {Van Hentenryck}, and Ferdinando Fioretto.
\newblock Bias and variance of post-processing in differential privacy.
\newblock In {\em Proceedings of the AAAI Conference on Artificial Intelligence
  {(AAAI)}}, page (to appear), 2021.

\end{thebibliography}
\appendix

\section{Missing Proofs}

\begin{proof}[Proof of Lemma \ref{lem:and_formula}]
The proof proceeds by cases.\\
Case \ref{lemma1:c1}:
$P^1_i(\bx) =\False; P^2_i(\bx) = \False$, therefore 
$P_i(\bx) = P^1_i(\bx) \land P^2_i(\bx) = \False$ and,
\begin{subequations}
\begin{align}
 % \begin{split}
     \Pr{P_i(\tilde{\bx})  \neq  P_i(\bx)}
     & = \Pr{ P^1_i(\tilde{\bx}) \land P^2_i(\tilde{\bx})  \neq \False}  \\ 
     & = \Pr{ P^1_i(\tilde{\bx}) \land P^2_i(\tilde{\bx})  = \True } \\
     & = \Pr{ P^1_i(\tilde{\bx}) = \True \land   P^2_i(\tilde{\bx}) = \True} \\
     \label{eq:p3_1}
     & = \Pr {P^1_i(\tilde{\bx}) = \True} \cdot  \Pr {P^2_i(\tilde{\bx}) = \True)}\\
     & = \Pr {P^1_i(\tilde{\bx}) \neq P^1_i(\bx)}  \cdot   
         \Pr {P^2_i(\tilde{\bx}) \neq P^2_i(\bx)} \\
     & = |B^i_{P^1}| |B^i_{P^2}| 
     % \end{split}
 \end{align}
 \end{subequations}
Where equation \eqref{eq:p3_1} is due to $P^1_i \independent P^2_i$.
 
\noindent 
Case \ref{lemma1:c2}: 
$P^1_i(\bx) =\False; P^2_i(\bx) = \True$, therefore $P_i(\bx) = P^1_i(\bx) \land P^2_i(\bx) = \False$, and
\begin{subequations}
 \begin{align}
      \Pr{P_i(\tilde{\bx})  \neq  P_i(\bx)}  
      & = \Pr{ P^1_i(\tilde{\bx}) \land P^2_i(\tilde{\bx}) \neq \False} \\  
      & = \Pr {P^1_i(\tilde{\bx}) \land P^2_i(\tilde{\bx})  = \True } \\
      & = \Pr {P^1_i(\tilde{\bx}) = \True \land  P^2_i(\tilde{\bx}) = \True } \\
      & = \Pr {P^1_i(\tilde{\bx}) = \True)} \cdot  \Pr {P^2_i(\tilde{\bx}) = \True)}\\
     \label{eq:p3_2}
      & = \Pr {P^1_i(\tilde{\bx}) \neq P^1_i(\bx)}  \cdot   
           \Pr{P^2_i(\tilde{\bx}) = P^2_i(\bx)} \\
      & = \Pr {P^1_i(\tilde{\bx}) \neq P^1_i(\bx)}  \cdot    
         \left( 1- \Pr{P^2_i(\tilde{\bx}) \neq P^2_i(\bx)} \right) \\
      & = |B^i_{P^1}| \left(1-|B^i_{P^2}|\right)
      % \end{split}
      \end{align}
 \end{subequations} 
 Where equation \eqref{eq:p3_2} is due to $P^1_i \independent P^2_i$.
     
Case \ref{lemma1:c3}:
$P^1_i(\bx) = \True; P^2_i(\bx) = \False$, therefore $P_i(\bx)) = P^1_i(\bx) \land P^2_i(\bx) = \False$, and 
\begin{subequations}
 \begin{align}
  % \begin{split}
      \Pr{P_i(\tilde{\bx})  \neq  P_i(\bx) }  
      & = \Pr {P^1_i(\tilde{\bx}) \land P^2_i(\tilde{\bx})  \neq \False} \\ 
      & = \Pr {P^1_i(\tilde{\bx}) \land P^2_i(\tilde{\bx})  = \True } \\
      & = \Pr {P^1_i(\tilde{\bx}) = \True \land   P^2_i(\tilde{\bx}) = \True} \\
      \label{eq:p3_3}
      & = \Pr {P^1_i(\tilde{\bx}) =\True} \cdot  \Pr{P^2_i(\tilde{\bx}) = \True}\\
      & = \Pr {P^1_i(\tilde{\bx}) = P^1_i(\bx)}  \cdot   \Pr{P^2_i(\tilde{\bx}) \neq P^2_i(\bx)}\\
      & = \left(1- \Pr{P^1_i(\tilde{\bx}) \neq P^1_i(\bx)} \right) \cdot   
                   \Pr{P^2_i(\tilde{\bx}) \neq P^2_i(\bx)}\\
      & = \left(1-|B^i_{P^1}|\right)|B^i_{P^2}|
      % \end{split}
    \end{align}
 \end{subequations} 
 Where equation \eqref{eq:p3_3} is due to $P^1_i \independent P^2_i$.

Case \ref{lemma1:c4}:
$P^1_i(\bx) = \True; P^2_i(\bx) = \True$, therefore $P_i(\bx) = P^1_i(\bx) \land P^2_i(\bx) = \True$, and 
\begin{subequations}
\begin{align}
 % \begin{split}
      \Pr{P_i(\tilde{\bx})  \neq  P_i(\bx)}  
     & = \Pr { P^1_i(\tilde{\bx}) \land P^2_i(\tilde{\bx})  \neq \True }  \\ 
     & = \Pr { P^1_i(\tilde{\bx}) \land P^2_i(\tilde{\bx})  = \False } \\
     & = 1 - \Pr{ P^1_i(\tilde{\bx}) = \True \land P^2_i(\tilde{\bx}) = \True} \\
     \label{eq:p3_4}
     & = 1 - \Pr{ P^1_i(\tilde{\bx}) = \True} \Pr{P^2_i(\tilde{\bx}) = \True}\\
     & = 1 - \left(1 - \Pr{P^1_i(\tilde{\bx}) \neq P^1_i(\bx)} \right) 
              \left(1 - \Pr{P^2_i(\tilde{\bx}) \neq P^2_i(\bx)}\right) \\
     & = 1 - \left(1 - |B^i_{P^1}|\right)\left(1-|B^i_{P^2}|\right) \\
     & = |B^i_{P^1}| +|B^i_{P^2}| - |B^i_{P^1}| |B^i_{P^2}|   \label{and_1}
     % \end{split}
 \end{align}
\end{subequations}
 Where equation \eqref{eq:p3_4} is due to $P^1_i \independent P^2_i$.
\end{proof}

\begin{proof}[Proof of Lemma \ref{lem:or_formula}]
The proof is similar to proof of Lemma \ref{lem:and_formula}.
\end{proof}

\begin{proof}[Proof of Lemma \ref{lem:xor_formula}]
The following hold for all four combination of binary boolean 
values for $P^1_i(\bx)_i, P^2_i(\bx) \in \{\False, \True\}$:
 \begin{align*}
    \Pr{P_i(\tilde{\bx})  \neq  P_i(\bx)}
    =& \Pr{P^1_i(\tilde{\bx}) \oplus P^2_i(\tilde{\bx}) \neq P^1_i(\bx) \oplus P^2_i(\bx)}\\
    =& 1 - \Pr{ P^1_i(\tilde{\bx}) = P^1_i(\bx)} \cdot   \Pr{P^2_i(\tilde{\bx}) = P^2_i(\bx)} \\
    -&     \Pr{ P^1_i(\tilde{\bx}) \neq P^1_i(\bx)}\cdot \Pr{P^2_i(\tilde{\bx}) \neq P^2_i(\bx)} \\
    =& 1 - |B^i_{P^1}||B^i_{P^2}| - (1-|B^i_{P^1}|)(1-|B^i_{P^2}|)  \\
    =&|B^i_{P^1}| + |B^i_{P^2}| -2|B^i_{P^1}| |B^i_{P^2}|.
 \end{align*}
  Where the second equality is due to $P^1_i \independent P^2_i$.
\end{proof}

%%%%%%%%%%%%%%%%%%%%%%%%%%%%%%%%%%%%%%%%%%%%%%%%%%%%%%%%%%%%%%%%%%%%%
\section{Experimental Details}
\label{sec:experimental_details}
%%%%%%%%%%%%%%%%%%%%%%%%%%%%%%%%%%%%%%%%%%%%%%%%%%%%%%%%%%%%%%%%%%%%%

\subsection{General Settings}
% In case when the level of unfairness $\alpha$ could not be computed analytically, we utilize the sampling method to estimate $\alpha$ by at least $10^4$ different random seeds in all experiments.

All experimental codes were written in Python 3.7. Some heavy computation tasks  were performed on a cluster equipped with Intel(R) 
Xeon(R) Platinum 8260 CPU @ 2.40GHz and 8GB of RAM. We will release our codes upon paper's acceptance.

\subsection{Datasets}
\paragraph{Title 1 School Allocation} The  dataset was  uploaded as a supplemental materials of \cite{abowd2019economic}. The dataset can be downloaded directly from \url{https://tinyurl.com/y6adjsyn}.
%true link is https://zenodo.org/record/1345775#.YAaXfpNKhAZ

We processed the dataset by removing schools which contains NULL information, and  keeping school districts with at least 1 students. The post-processed dataset left with 16441 school districts.

\paragraph{Minority language voting right benefits}
The dataset can be downloaded from 
\url{https://tinyurl.com/y2244gbt}. 
%https://www.census.gov/data/datasets/2016/dec/rdo/section-203-determinations.html
% There're three dataset files and one description file
% Section-203-Public-File-Documentation-20161201-v2.pdf in the folder.
% All the details about dataset are clearly described in the description
% file and we used the data in sect203\_All\_Areas.csv for experiments.
% Note that there's something wrong in dataset file
% sect203\_All\_Areas.csv, the value of column 'VACLIT' should be equal
% to (100 * 'VACLEP' / 'LEPPCT') but not in the original file. So we
% need to process the original dataset and create a new correct dataset
% for experiments.

The focus of the experiments is on Hispanic groups, which represent 
the largest minority population. There are 2774 counties that contain 
at least a Hispanic person. 
% For post-processing after adding noise to 
% the variable, we assume that value of 'ILLIT'($x^{spe}$) should be 
% larger than 0 and value of 'VACIT'($x^s$), 'VACLEP'($x^{sp}$) should 
% be larger than 1.

% \paragraph{Apportionment of Legislative Representatives} 
% The dataset can be downloaded directly from \url{https://tinyurl.com/y4xswrtc}. The dataset provides state-level population for 35 Indian states for different years. Follow the study in \cite{pujol:20} experiments on this paper focuses only on data from 1971. 
%https://www.indiabudget.gov.in/budget_archive/es2006-07/chapt2007/tab97.pdf

\subsection{Mechanism Implementation}
\paragraph{Linear Proxy Allocation $\bar{P}^F$}
The linear proxy allocation method used for problem $\bar{P}^F$ is implement so that, for a given privacy parameter $\epsilon$, the 
algorithm allocates $\epsilon_1 = \frac{\epsilon}{2}$ to release the normalization term $Z$. 
The remaining $\epsilon_2 = \frac{\epsilon}{2}$ budget is used to publish the population counts $x_i$.

\paragraph{Output Perturbation mechanism}
The paper uses standard Laplace mechanism: 
$\tilde{P}^F_i(\bx) = P_i(x) + \mbox{Lap}(0,\frac{\Delta}{\epsilon})$. Therein, the global sensitivity $\Delta$ is obtained from Theorem \ref{thm:sensitivity_school}. 
The experiments set a known public lower bound $L = 0.9 Z$ for the normalization term $Z$ in Theorem \ref{thm:sensitivity_school}. %Likewise for the Apportionment of Legislative Representatives Problem, we employ the Laplace mechanism where the sensitivity is from Theorem \ref{thm:sens_rep}, and a prior knowledge over the lower bound for Indian population $L = 0.9Z$ is given.

% \paragraph{Sensitivity for Output Perturbation methods}
% The following discussion provides the  global sensitivity for output perturbation method for each problem presented in the paper.
\begin{theorem}
\label{thm:sensitivity_school}
Denote $a_{\max} = \max_i a_i$, and let $L \leq \sum_{i \in [n]} x_i a_i $ is a known public lower bound for the normalization term. The $l_1$ global sensitivity of the query $P^F = \{P^F_i \}^n_{i=1}$ with $P^F_i= \big( \frac{x_i a_i}{\sum_{i \in [n]} x_i a_i}\big)$  is given by:
\begin{equation}
    \Delta = \max_{\bm{x},\bm{x'}} \left| P^F(\bm{x}) - P^F(\bm{x'}) \right|_1 = \frac{2 a_{\max}}{L}
\end{equation}
\end{theorem}

\begin{proof}
Let $\bm{x}'$ be a dataset constructed by removing a single individual from $\bm{x}$ and denote with $Z = \sum_{j}  x_j a_j$. 
It follows that:
\begin{align*}
P^F_k(\bm{x})  - P^F_k(\bm{x'}) =
\begin{cases}
        \frac{x_k a_k}{Z}  - \frac{(x_k -1)  a_k }{Z - a_k } & \text{ if } k = i\\
        \frac{x_k a_k}{Z} - \frac{x_k a_k   }{Z - a_k } & \text{ otherwise}.
  \end{cases}
\end{align*}
When $k=i$, it follows that:

\begin{subequations}
\begin{align}
 P^F_i(\bm{x})  - P^F_i(\bm{x'}) 
 &= \frac{a_i \left(Z - x_i a_i\right)}{ Z(Z-a_i )}\\
 & \leq \frac{a_i (Z - a_i)}{ Z(Z-a_i )} = \frac{a_i }{ Z} \leq \frac{a_{\max}}{Z}\\
 & \leq \frac{a_{\max}}{L} 
 \label{eq:fund_first_bound}
\end{align}
\end{subequations}
The last inequality is due to assumption that $Z  = \sum _{j \in [n]} a_j\cdot x_j \geq L$.

Next, when $k \neq i$:
\begin{align}
 P^F_j(\bm{x}) - P^F_j(\bm{x'}) =  \frac{-a_jx_j a_i}{Z(Z- a_i)},
\end{align}
and thus
\begin{subequations}
\begin{align}
 &\sum_{j \neq i} \left|P^F_j(\bm{x}) - P^F_j(\bm{x'})\right| = a_i \frac{Z -a_ix_i}{Z(Z-a_i)}\\
 & \leq  a_i \frac{Z -a_i}{Z(Z-a_i)} = a_i\frac{1}{Z} \leq \frac{a_{\max}}{L}
 \label{eq:fund_second_bound}
\end{align}
\end{subequations}
The bound is obtained by adding Equation \eqref{eq:fund_first_bound} with Equation \eqref{eq:fund_second_bound}.
\end{proof}

\end{document}